\documentclass[10pt,twocolumn,letterpaper]{article}

\usepackage[pagenumbers]{cvpr} % To force page numbers, e.g. for an arXiv version

\makeatletter
\@namedef{ver@everyshi.sty}{}
\makeatother
\usepackage{graphicx}
\usepackage{amsmath}
\usepackage{amssymb}
\usepackage{booktabs}
\usepackage{booktabs}
\usepackage{amsthm}
\usepackage{amsfonts}
\usepackage{color}
\usepackage{subcaption}
\usepackage{pgfplots}
\usepgfplotslibrary{groupplots}
\pgfplotsset{compat=newest}
\usepackage{tikz}
\usetikzlibrary{matrix}
\usepackage{stackengine}
\usepackage{mwe}

\newcommand{\cD}{\mathcal{D}}

\newcommand{\cL}{\mathcal{L}}

\newcommand{\cX}{\mathcal{X}}
\newcommand{\cY}{\mathcal{Y}}

\newcommand{\bP}{\mathbb{P}}
\newcommand{\hx}{\hat{x}}

\newcommand{\ga}{\alpha}
\newcommand{\ft}{f_\theta}
\newcommand{\vect}[1]{\boldsymbol{\mathbf{#1}}}
\newcommand{\norm}[1]{\lVert #1 \rVert}

\DeclareMathOperator*{\argmax}{arg\,max}
\DeclareMathOperator*{\argmin}{arg\,min}
\usepackage{algorithmicx}
\usepackage{algorithm}
\usepackage{algpseudocode}

\usepackage{empheq}

\theoremstyle{definition}
\newtheorem{definition}{Definition}
\newtheorem{theorem}{Theorem}
\newtheorem{lemma}[theorem]{Lemma}
\usepackage{enumitem}
\usepackage[para]{footmisc}
\usepackage[pagebackref,breaklinks,colorlinks]{hyperref}

\usepackage[capitalize]{cleveref}
\crefname{section}{Sec.}{Secs.}
\Crefname{section}{Section}{Sections}
\Crefname{table}{Table}{Tables}
\crefname{table}{Tab.}{Tabs.}

\def\cvprPaperID{5487} % *** Enter the CVPR Paper ID here
\def\confName{CVPR}
\def\confYear{2022}

\begin{document}

\title{Interpolated Joint Space Adversarial Training for Robust and Generalizable Defenses}
\author{Chun Pong Lau\textsuperscript{1}, Jiang Liu \textsuperscript{1}, Hossein Souri\textsuperscript{1}, Wei-An Lin\textsuperscript{2}, Soheil Feizi\textsuperscript{3}, Rama Chellappa\textsuperscript{1} \\
\textsuperscript{1}Johns Hopkins University, \textsuperscript{2}Adobe, \textsuperscript{3}University of Maryland, College Park \\
{\tt\small \{jliu214, clau13, hsouri1, rchella4\}@jhu.edu, walin@umd.edu, sfeizi@cs.umd.edu}}
% \author{Chun Pong Lau\thanks{Johns Hopkins University}\\
% {\tt\small clau13@jhu.edu}
% \and
% Jiang Liu\footnotemark[1]\\
% {\tt\small jiangliu@jhu.edu}
% \and
% Hossein Souri\footnotemark[1]\\
% {\tt\small hsouri1@jhu.edu}
% \and
% Wei-An Lin\thanks{Adobe}\\
% {\tt\small walin@umd.edu}
% \and
% Soheil Feizi\thanks{University of Maryland, College Park}\\
% {\tt\small  sfeizi@cs.umd.edu}
% \and
% Rama Chellappa\footnotemark[1]\\
% {\tt\small rchella4@jhu.edu}
% }
\maketitle

%%%%%%%%% ABSTRACT
\begin{abstract}
\vspace{-3mm}
Adversarial training (AT) is considered to be one of the most reliable defenses against adversarial attacks. However, models trained with AT sacrifice standard accuracy and do not generalize well to novel attacks. Recent works show generalization improvement with adversarial samples under novel threat models such as on-manifold threat model or neural perceptual threat model. However, the former requires exact manifold information while the latter requires algorithm relaxation. Motivated by these considerations, we exploit the underlying manifold information with Normalizing Flow, ensuring that exact manifold assumption holds. Moreover, we propose a novel threat model called Joint Space Threat Model (JSTM), which can serve as a special case of the neural perceptual threat model that does not require additional relaxation to craft the corresponding adversarial attacks. Under JSTM, we develop novel adversarial attacks and defenses. The mixup strategy improves the standard accuracy of neural networks but sacrifices robustness when combined with AT. To tackle this issue, we propose the Robust Mixup strategy in which we maximize the adversity of the interpolated images and gain robustness and prevent overfitting. Our experiments show that Interpolated Joint Space Adversarial Training (IJSAT) achieves good performance in standard accuracy, robustness, and generalization in CIFAR-10/100, OM-ImageNet, and CIFAR-10-C datasets. IJSAT is also flexible and can be used as a data augmentation method to improve standard accuracy and combine with many existing AT approaches to improve robustness.
  \vspace{-5mm}
\end{abstract}

%%%%%%%%% BODY TEXT
\section{Introduction}
\setlength{\abovecaptionskip}{3pt}
\setlength{\belowcaptionskip}{3pt}
\setlength{\textfloatsep}{10pt}% Remove \textfloatsep
%please define on-manifold in the introduction section

 Although recent research has shown that neural networks are highly successful in various applications \cite{girshick2015fastrcnn,hinton2012acoustic,levine2020wasserstein,imagenet2012}, they are vulnerable to adversarial samples which are intentionally designed to fool the model \cite{goodfellow2015explaining,carlini2017towards,athalye2018obfuscated,fawzi2018,athalye2018synthesizing}. This poses a massive challenge in security-critical applications such as autonomous driving \cite{kurakin2016adversarial} and medical imaging \cite{antun2020instabilities}. 

% \begin{figure}[t] 
% \centering

% \begin{subfigure}[t]{0.15\textwidth}
% \includegraphics[width=\textwidth]{png/Fig1/I_clean.png}
% \caption{Clean image}
% \end{subfigure}
% \begin{subfigure}[t]{0.15\textwidth}
% \includegraphics[width=\textwidth]{png/Fig1/I_l2.png}
% \caption{Image attacked by PGD}
% \end{subfigure}
% \begin{subfigure}[t]{0.15\textwidth}
% \includegraphics[width=\textwidth]{png/Fig1/image_l2_diff_200_3.png}
% \caption{Magnified difference between (a) and (b)}
% \end{subfigure}
% \\
% \begin{subfigure}[t]{0.15\textwidth}
% \includegraphics[width=\textwidth]{png/Fig1/I_clean.png}
% \caption{Clean image}
% \end{subfigure}
% \begin{subfigure}[t]{0.15\textwidth}
% \includegraphics[width=\textwidth]{png/Fig1/I_dual.png}
% \caption{Image attacked by DPGD}
% \end{subfigure}
% \begin{subfigure}[t]{0.15\textwidth}
% \includegraphics[width=\textwidth]{png/Fig1/image_double_diff_200_3.png}
% \caption{Magnified difference between (d) and (e)}
% \end{subfigure}

% \caption{Comparsion between PGD and DPGD attacks applied on a clean image. Note that (b) and (e) have the same $L_2$ difference from (a).} %The adversarial perturbation was produced using the proposed Dual-PGD attack with maximum perturbation $\epsilon=4$ (out of 256) in the image space and $\eta=0.02$ in the latent space. In this example, the adversarial image is incorrectly recognized as “monkey”; the true label is “feline”.}
% \label{fig: DMT attack demo}
% \end{figure}

\begin{figure}[t!]
    \centering
    \captionsetup[subfigure]{justification=centering, belowskip=0pt}
    %  \begin{subfigure}[t]{0.25\textwidth}
    %     \raisebox{-12ex}{\rotatebox[origin=c]{90}{\text{PGD-50}}}
    % \end{subfigure}
    % \hspace{-3cm}
    \begin{subfigure}[t]{0.156\textwidth}
        \raisebox{-\height}{\includegraphics[width=\textwidth]{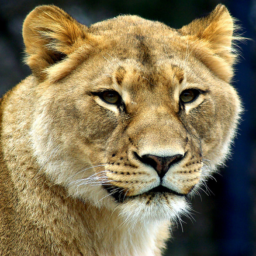}}
        \caption{Original}
    \end{subfigure}
    \hspace{-1.5mm}
    \begin{subfigure}[t]{0.078\textwidth}
        \raisebox{-\height}{\includegraphics[width=0.99\textwidth]{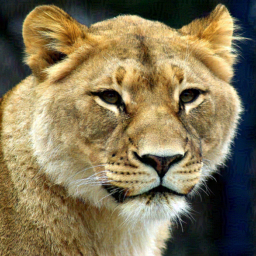}}\\
        \raisebox{-\height}{\includegraphics[width=0.99\textwidth]{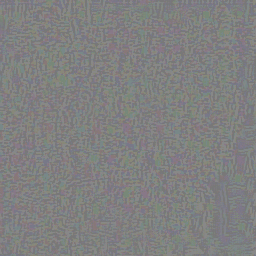}}
        \caption{PGD \cite{madry2017towards}}
    \end{subfigure}
    \hspace{-1.5mm}
    \begin{subfigure}[t]{0.078\textwidth}
        \raisebox{-\height}{\includegraphics[width=0.99\textwidth]{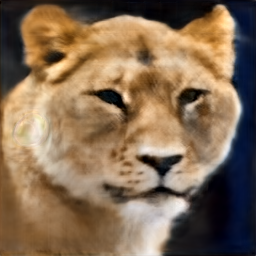}}\\
        \raisebox{-\height}{\includegraphics[width=0.99\textwidth]{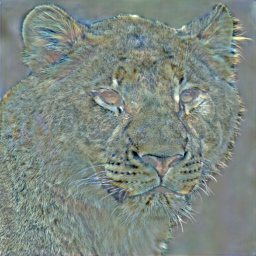}}
        \caption{OM-PGD (GAN) \cite{lin2020dual}}
    \end{subfigure}
    \hspace{-1.5mm}
    \begin{subfigure}[t]{0.078\textwidth}
        \raisebox{-\height}{\includegraphics[width=0.99\textwidth]{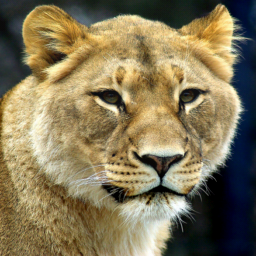}}\\
        \raisebox{-\height}{\includegraphics[width=0.99\textwidth]{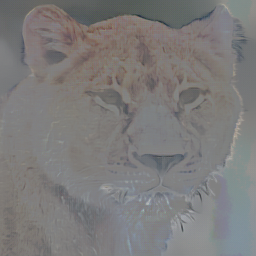}}
        \caption{OM-PGD (Flow)}
    \end{subfigure}
    \hspace{-1.5mm}
    \begin{subfigure}[t]{0.078\textwidth}
        \raisebox{-\height}{\includegraphics[width=0.99\textwidth]{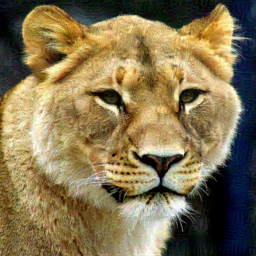}}\\
        \raisebox{-\height}{\includegraphics[width=0.99\textwidth]{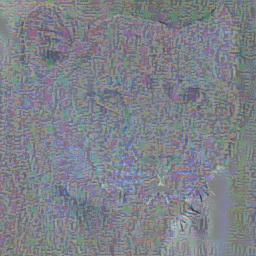}}
        \caption{JSA}
    \end{subfigure}
    
    \caption{Visualization of adversarial samples from AT \cite{madry2017towards}, OM-PGD (GAN) \cite{lin2020dual}, OM-PGD (Flow) and our proposed JSA.} \label{fig:overview}
    % \vspace{-5mm}
    
\end{figure}

% \begin{filecontents*}{clean_acc.csv}
% Epoch, Normal, AT, IAT, DMAT, Double, RDMAT, RIDMAT
% 2,48.95,45.21,43.81,44.46,23.21,38.79,40.05
% 4,53.34,55.62,55.1,53.32,38.97,48.96,49.46
% 6,71.77,65.89,66.42,69.58,51.31,64.19,57.61
% 8,71.83,67.16,71.23,71.33,57.17,66.91,62.64
% 10,71.24,70.9,72.66,73.7,58.66,68.4,67.82
% 12,74.92,72.9,75.37,76.48,64.78,72.57,70.32
% 14,74.83,73.53,75.91,77.74,67.47,71.25,72.14
% 16,74.79,72.99,76.75,77.38,68.04,73.14,72.66
% 18,74.76,73.12,75.81,77.74,69.04,73.35,72.88
% 20,74.72,73.31,76.81,77.96,68.4,74.72,73.72
% \end{filecontents*}

% \begin{filecontents*}{adv_acc.csv}
% Epoch, Normal, AT, IAT, DMAT, Double, RDMAT, RIDMAT
% 2,0,31.81,25.46,28.61,19.13,28.88,30.51
% 4,0,36.91,28.14,32.21,31.77,34.4,36
% 6,0,44.95,34.81,42.05,38.77,44.86,41.76
% 8,0,42.76,36.57,40.22,42.99,46.75,45.17
% 10,0,45.69,38.8,42.33,43.82,48.54,49.55
% 12,0,45.65,38.09,41.12,46.53,51.59,50.78
% 14,0,42.13,39.26,38.99,48.9,48.69,51.23
% 16,0,40.16,37.61,37.41,48.94,49.3,50.75
% 18,0,39.77,35.64,37.75,49.22,49.4,50.74
% 20,0,38.88,37.15,37.86,48.83,49.42,50.82
% \end{filecontents*}

\begin{figure}[t!] 
\centering
\begin{tikzpicture}
\begin{groupplot}[
group style={group size=1 by 1},
width=0.5\textwidth,
ymin=0,ymax=60,
xmin=0,xmax=120,
xlabel={Epoch},
]

\nextgroupplot[title={Robust Accuracy against PGD Attack},
style={line width=0.5pt},
grid=both,
grid style={line width=.1pt, draw=gray!40},
ylabel=Accuracy,
y label style={at={(axis description cs:0,.5)},anchor=south}
]

\addplot[black, thick] table[x=Epoch, y=clean, col sep=comma,] {png/plot/adv_acc_new.csv}; \label{normal}

\addplot[red, thick] table[x=Epoch, y=AT, col sep=comma,] {png/plot/adv_acc_new.csv}; \label{AT}

\addplot[blue, thick] table[x=Epoch, y=jsa_v2, col sep=comma,] {png/plot/adv_acc_new.csv}; \label{jsa}

\addplot[green, thick] table[x=Epoch, y=imat_mixup_v2, col sep=comma,] {png/plot/adv_acc_new.csv}; \label{imat_mixup}

\addplot[brown, thick] table[x=Epoch, y=imat_v2_02, col sep=comma,] {png/plot/adv_acc_new.csv}; \label{imat}

% \addplot[black, mark=+, mark size=1.5pt] table[x=Epoch, y=Normal, col sep=comma,] {adv_acc.csv};
% \addplot[red, mark=+, mark size=1.5pt] table[x=Epoch, y=AT, col sep=comma,] {adv_acc.csv}; 
% \addplot[orange, mark=+, mark size=1.5pt] table[x=Epoch, y=IAT, col sep=comma,] {adv_acc.csv};
% \addplot[blue, mark=+, mark size=1.5pt] table[x=Epoch, y=DMAT, col sep=comma,] {adv_acc.csv}; 
% \addplot[green, mark=diamond*, mark size=1.5pt] table[x=Epoch, y=Double, col sep=comma,] {adv_acc.csv}; 
% \addplot[purple, mark=diamond*, mark size=1.5pt] table[x=Epoch, y=RDMAT, col sep=comma,] {adv_acc.csv};
% \addplot[brown, mark=diamond*, mark size=1.5pt] table[x=Epoch, y=RIDMAT, col sep=comma,] {adv_acc.csv};
\coordinate (top) at (rel axis cs:0,1);
\coordinate (bot) at (rel axis cs:1,0);
\end{groupplot}
    
legend
  \path (top|-current bounding box.south)--
        coordinate(legendpos)
        (bot|-current bounding box.south);
    % \path coordinate(legendpos)
    %     (bot|-current bounding box.south);
  \matrix[
      matrix of nodes,
      anchor=north,
      draw,
      inner sep=0.2em,
    %   column 1/.style={nodes={align=center}},
    %   column 2/.style={nodes={anchor=base west, font=\tiny}},
    %   column 3/.style={nodes={align=center}},
    %   column 4/.style={nodes={anchor=base west, font=\tiny}},
    ]at([yshift=20ex]legendpos)
    { \ref{normal}& Normal&[5pt]
      \ref{AT}& AT[PGD] &[5pt] \\
      \ref{jsa}& AT[JSA] &[5pt] 
      \ref{imat_mixup}& AT[JSA] + mixup &[5pt] \\
      \ref{imat}& IJSAT &[5pt] \\};
\end{tikzpicture}
\vspace{-2mm}
\caption{PGD robustness during training with CIFAR-10 dataset. Using JSA with AT, we can observe a significant improvement. Using normal mixup with it, we can see further improvement but also observe robust overfitting after the learning rate changes (at $75^{th}$ epoch). IJAST has better robustness and does not have robustness drop after learning rate changes.}
\label{fig:cifar10 plot}
% \vspace{-5mm}
\end{figure}
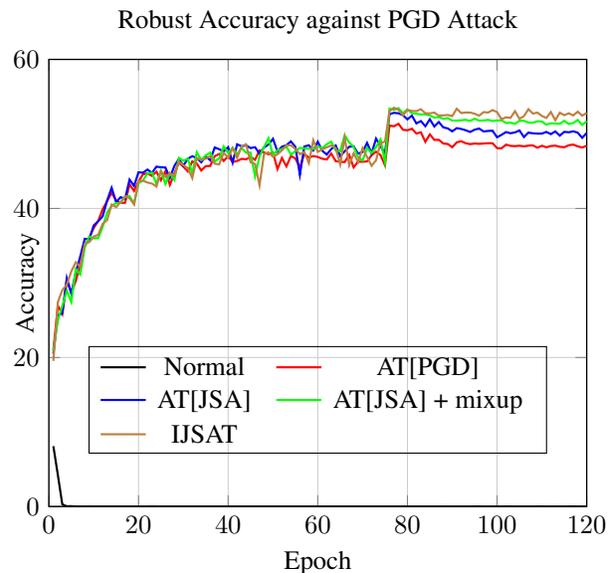

\begin{figure*}[t!]
  \centering
%   \fbox{\rule[-.5cm]{0cm}{5cm} \rule[-.5cm]{12cm}{0cm}}
  
  \includegraphics[width=0.95\textwidth]{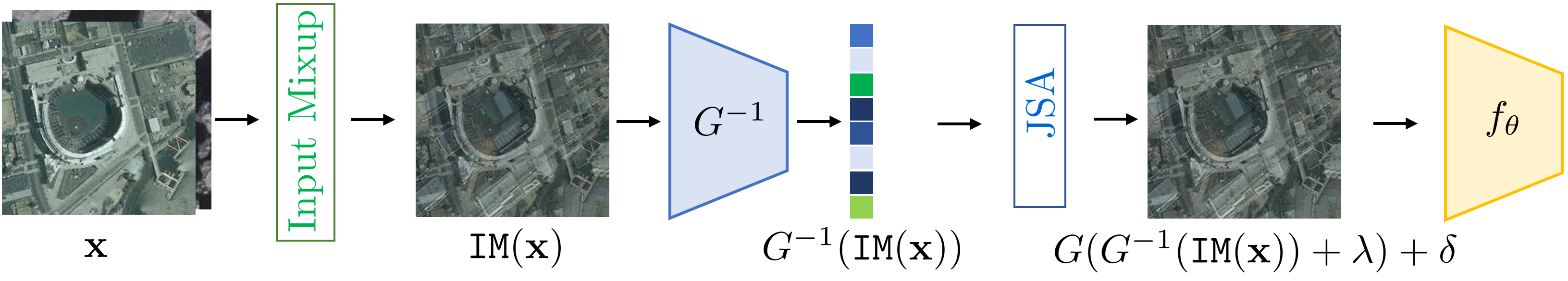}
  \caption{The overall pipeline of the proposed Interpolated Joint Space Adversarial Training (IJSAT). First, Input Mixup is applied to two images to get an interpolated image. Then it is passed to the Flow-based Model to get the latent vector. We apply the proposed Joint Space Attack to it and get the adversarial sample. Finally, we pass it to the classifier. During test time, images are directly passed to the adversarially trained classifier.} \label{fig:framework}
  \vspace{-7mm}
\end{figure*}
To tackle this problem, several defense methods have been proposed including empirical and certifiable defense methods \cite{xiao2019resisting,roth2019odds,pang2019mixup,samangouei2018defense,wong2018provable,raghunathan2018semidefinite,cohen2019certified,levine2019robustness,chiang2020certified,levine2020wasserstein, Stutz2020ICML}. In Adversarial Training, the defender generates adversarial samples that obey a certain threat model and utilizes them in the training process in order to get a robust model \cite{madry2017towards}. In other words, existing AT methods consider perturbed samples, for example, within a small $L_p$ norm-bound distance and assure robustness to the same type of perturbed samples. However, models trained using AT suffer from not being robust to novel imperceptible attacks \cite{laidlaw2021perceptual,laidlaw2019functional,song2018adv,poursaeed2019finegrained}. In addition, it has been observed that AT can often cause a reduction in standard accuracy, i.e., the accuracy on clean data, which indicates that a trade-off between robustness and accuracy \cite{madry2017towards,zhang2019theoretically,Stutz_2019_CVPR,raghunathan2020understanding} is at play. To further improve robustness, one usually trains models by AT with a larger attack budget. Although this improves robustness, \textit{the standard accuracy of the models is sacrificed}.

%define on-manifold

Recently, several works aim to use a low-dimensional underlying data manifold to attack neural networks by creating on-manifold adversarial samples \cite{jalal2017robust,song2018adv,Stutz_2019_CVPR, lin2020dual}. On-manifold adversarial samples are adversarial samples constrained to lie on data manifolds and are obtained by perturbing inputs in the latent space learned by generative models. Adversarial samples computed in the image space are considered as off-manifold \cite{Stutz_2019_CVPR}. On-manifold adversarial samples have been used to break models trained by AT \cite{song2018adv} as well as deep generative model-based defense methods \cite{Chen2020onbreaking} such as DefenseGAN \cite{samangouei2018defense}, Analysis by Synthetics \cite{schott2018towards} and MoG-VAE \cite{ghosh2019resisting}. In response to this attack, several manifold-based defense methods have been proposed lately \cite{jalal2017robust,Stutz_2019_CVPR, lin2020dual}. \cite{jalal2017robust} takes advantage of the low-dimensional data manifold and introduces a sup player, which is more powerful than regular adversarial training. \cite{Stutz_2019_CVPR} shows that robustness against on-manifold adversarial samples is related to the generalization ability of neural networks and proposes on-manifold adversarial training. \cite{lin2020dual} investigates the scenario that the manifold hypothesis holds (i.e. exact information about the underlying data manifold is given) and proposes dual manifold adversarial training (DMAT), which utilizes both on-manifold and off-manifold adversarial samples to robustify the models. The authors show that DMAT can improve model robustness against $L_p$ and non-$L_p$ attacks. However, \textit{DMAT only considers on-manifold images during both training and testing} while in practice, the testing images are natural (off-manifold) images. 

Interpolation-based training has been recently proposed in several works \cite{zhang2018mixup,pang2019mixup,verma2019manifold,lamb2019interpolated,lee2020adversarial}, which can improve the generalization and robustness properties of neural networks. \cite{zhang2018mixup} has introduced a data augmentation routine called \textit{input mixup} in which two data samples are drawn from the training dataset, and a linear interpolation of the samples is passed through the network with the loss computed using the same linear interpolation of losses. The \textit{manifold mixup} as a regularization tool uses the interpolation of the hidden representations of a randomly selected layer \cite{verma2019manifold}. \textit{Interpolated adversarial training} (IAT) \cite{lamb2019interpolated} proposes to employ either \textit{input mixup} or \textit{manifold mixup} combined with AT to benefit from both interpolated and AT methods and improves generalization while preserving robustness. However, IAT interpolates images after perturbation such that \textit{the interpolated images are not guaranteed to mislead the classifier}.

In this paper, we focus on developing a classifier with good standard accuracy, robust to seen attacks, and generalizing well to unseen attacks. To overcome the inability of $L_p$ adversarial samples to maintain standard accuracy and the on-manifold adversarial samples' ineffectiveness in defending the $L_p$ attacks, we propose a novel threat model called \textit{Joint Space Threat Model (JSTM)} that perturbs images in both image space and latent space (See Fig.~\ref{fig:overview}(e)). To ensure that on-manifold adversarial samples generalize well, which requires the exact manifold assumption, we use an \textit{invertible Flow-based model} to guarantee generalizability. (See Fig.~\ref{fig:overview}(d)). To prevent robust overfitting \cite{rice2020overfitting} and further improve robustness generalization, we propose a \textit{Robust Mixup strategy} in which we attack the interpolated samples directly to increase their adversity and hence achieve better robustness (See Fig.~\ref{fig:cifar10 plot}). In light of these, we propose \textbf{I}ntepolated \textbf{J}oint \textbf{S}pace \textbf{A}dversarial \textbf{T}raining (\textbf{IJSAT}), which applies Robust Mixup strategy and trains the model with JSA samples. The overall pipeline is shown in Fig.~\ref{fig:framework}.

% \vspace{-3mm}
\subsubsection*{Contributions:}
% \noindent \textbf{Contribution:}
% \vspace{-1mm}
\begin{itemize}[leftmargin=3.5mm]
    \item We propose Joint Space Attack (JSA), which simultaneously optimizes the perturbations in image space and latent space. We empirically show JSA samples are robust to both $L_p$ attacks and unseen attacks. \vspace{-2mm}
    \item We propose the Robust Mixup strategy, which can further improve robustness and generalization and prevent overfitting.\vspace{-2mm}
    \item We empirically demonstrate that IJAST achieves good performance in both standard accuracy, robustness, and generalization in CIFAR-10, CIFAR-100, OM-ImageNet, and CIFAR-10-C datasets. \vspace{-2mm}
    \item IJAST can also serve as a data augmentation method to improve the standard accuracy of the classifiers and assist with other AT methods to achieve better robustness. 
\end{itemize}

\section{Mathematical Background}
\subsection{Setup}
We consider the multi-class classification problem, where the image samples $x\in \cX := \mathbb{R}^{H \times W \times C}$ are drawn from an underlying distribution $\mathbb{P}_X$, where $H$, $W$ and $C$ are the height, width and the number of channels of the image respectively. Let $f_\theta$ be a parameterized model which maps any image in $\cX$ to a discrete label $y$ in $\cY:= \{1, \cdots, |\cY| \}$. An \emph{accurate} classifier maps an image $x$ to its corresponding true label $y_{\text{true}}$, i.e. $f_\theta(x) = y_\text{true}$. A successful attack fools the classifier to map an adversarial image $\hx$ to a wrong label, i.e. $f_\theta(\hx) \neq y_\text{true}$. We define the manifold information to be \emph{exact}, when there exists a generative model $G$ such that there exists a latent representation $z$ for every image $x$. 
% Without loss of generality, we assume $\mathbb{P}_X$ is supported on a lower-dimensional manifold $\cM$ and can be approximated by $\mathbb{P}_{G(Z)}$ where $G(\cdot)$ is a generative model and $Z$ is distributed according to a Gaussian distribution. Informally, we refer to the support of $G(Z)$ as the approximate image manifold $\bar{\cM}$ for $\cM$. We say the manifold information is \emph{exact}, when there exists a generative model $G$ such that $\cM = \bar{\cM}$. 
\begin{definition}[Exact Manifold Assumption]
    Fix a generator G. $$\forall x\in \cX,~\exists z\in Z~s.t.~G(z)=x.$$
\end{definition}
% \noindent \textbf{Remarks: }To serve our purpose, we only require this assumption holds with the images in the dataset. 
% \Andy{Maybe create a new subsection here to discuss threat models and attacks}
\subsection{Standard and On-Manifold Robustness}

We consider off-manifold adversarial samples $\hx^{img}$ and on-manifold adversarial samples $\hx^{lat}$ in the context of standard adversarial robustness and on-manifold adversarial robustness respectively. Both $\hx^{img}$ and $\hx^{lat}$ are visually indistinguishable from $x$. To create perturbations in image space $\cX$ and latent space $Z$, we consider the popular $L_p$ additive attacks where $\hx^{img} = x + \delta$ and $\hx^{lat} = G(z + \lambda)$ where $z \in Z $ is the corresponding latent vector. Formally, 
\begin{equation}
    \max_{\delta \in \Delta} \cL (f_{\theta}(x + \delta), y_{\text{true}}),
    \label{eq:regular}
\end{equation}
and 
\begin{equation}
    \max_{\lambda \in \Lambda} \cL (f_{\theta}(G(z + \lambda)), y_{\text{true}}),
    \label{eq:on-manifold}
\end{equation}
where $\Delta = \{ \delta: \norm{\delta}_p < \epsilon\}$, $\Lambda = \{ \lambda: \norm{\lambda}_p < \eta\}$ and $\cL$ is a classification loss function (e.g. the cross-entropy loss). In our work, we focus on $p = \infty$, and will explicitly specify when $p \neq \infty$ . In \eqref{eq:regular}, since the function is non-convex, the maximization is typically performed using gradient-based optimization methods. In this paper, we consider the PGD attack and use the notation PGD-$K$ to represent $K$-step PGD attacks with bounded $L_\infty$ norm. 

To defend against norm-bounded attacks, an established AT approach by Madry \etal~\cite{madry2017towards} considers the following min-max formulation:
\begin{equation}
    \min_{\theta} \sum_i \max_{\delta \in \Delta} \cL(f_\theta(x_i + \delta), y_{\text{true}}),
    \label{eq:adv_train_standard}
    % \vspace{-1mm}
\end{equation}
where the classification model $f_{\theta}$ is trained exclusively on adversarial images by minimizing the cross-entropy loss. In a similar manner, we can defend the on-manifold adversarial samples which are crafted by OM-PGD attack by 

\begin{equation}
    \min_\theta \sum_i \max_{\lambda \in \Lambda} \cL (f_{\theta}(G(z_i + \lambda)), y_{\text{true}}).
    \label{eq:adv_train_manifold}
\end{equation}
% This approach is called the on-manifold adversarial training (OM-AT). \cite{lin2020dual} combines these two, and solves for

% \begin{equation}
%     \begin{split}
%         \min_{\theta} &\sum_i \Big\{ \max_{\delta \in \Delta} \cL(f_\theta(x_i + \delta), y_{\text{true}}) \\ 
%         &+ \max_{\lambda \in \Lambda} \cL(f_\theta(G(z_i + \lambda)), y_{\text{true}}) \Big\},
%     \end{split}
%     \label{eq:adv_train_dual}
% \end{equation}

% \noindent where $x_i$ is the natural clean image and $x_i = G(z_i)$ if exact manifold information is given. 
% The most straightforward approach to combine the two perturbations, namely the image space perturbation and the latent space perturbation, is proposed by DMAT which feeds both adversarial samples to the classifier. 

\subsection{Flow-Based Generative Model}
Suppose $Z$ is a random variable with an explicit and tractable probability function (pdf) $p_Z : Z \rightarrow \mathbb{R}$. Let $G$ be an invertible function and $X = G(Z)$. By change of variables equation, the pdf of $X$:
\begin{equation} \label{eq: flow}\
\begin{split}
    p_X(x) &= p_Z(G^{-1}(x))~|\text{det D}G^{-1}(x)| \\
    &= p_Z(G^{-1}(x))~|\text{det D}G(G^{-1}(x))|^{-1},
\end{split}
\end{equation}

\noindent where $G^{-1}$ is the inverse of $G$, D$G(z) = \frac{\partial G}{\partial z}$ is the Jacobian of $G$ and D$G^{-1}(x) = \frac{\partial H}{\partial x}$ is the Jacobian of $G^{-1}$.

Although GANs largely dominate generative models, one major drawback of GANs is that they cannot compute the exact sample likelihoods. Flow-based generative models solve this problem by having an invertible function that has the properties mentioned above. The invertible function $G$ is typically modeled as the composition of $K$ invertible maps, i.e. $G = G_1 \circ G_2 \circ \cdots \circ G_K$. This is also called the \textit{normalizing flows} \cite{DBLP:journals/corr/DinhKB14}. Different designs for constructing $G$ have been proposed in recent years. NICE \cite{DBLP:journals/corr/DinhKB14} uses coupling layers to enable highly expressive transformation for flows. RealNVP \cite{DBLP:conf/iclr/DinhSB17} uses affine coupling layers which basically represent invertible scale transformations. Glow \cite{DBLP:conf/nips/KingmaD18} uses invertible $1 \times 1$ convolutions to have learnable permutations.

\subsection{Mixup Strategy}
Overfitting usually occurs when models are adversarially trained. This leads to the degradation of standard accuracy and generalization. To mitigate this drawback while preserving certain robustness, interpolation-based training techniques \cite{zhang2018mixup,pang2019mixup,verma2019manifold,lamb2019interpolated} are adopted, which are shown to be effective and have demonstrated promising performance in adversarial robustness and generalization improvements. In our work, we focus on using Input Mixup \cite{zhang2018mixup} to combine the off-manifold and on-manifold adversarial samples. Input Mixup draws two pairs of data sample from the dataset, $(x_i, y_i) $, $(x_j, y_j) \sim \bP_X$, and takes the convex combination between them in the image space $\mathtt{IM}(\vect{x}) = \ga x_i + (1 - \ga)x_j$, where $\mathtt{IM}$ is the Input Mixup function, $\vect{x} = (x_i, x_j)$ is the pair of images and $\ga \in (0, 1)$ is a random variable with Beta distribution. The interpolated image $\mathtt{IM}(\vect{x})$ is passed into the classifier $\ft$ with minimizing the convex combination of the cross-entropy loss,
% \vspace{-2mm}
\begin{equation}
    \resizebox{0.9\hsize}{!}{$\cL^{mix}_\ga(\mathtt{IM}(\vect{x})) = \ga \cL(\ft(\mathtt{IM}(\vect{x})), y_i) + (1 - \ga)\cL(\ft(\mathtt{IM}(\vect{x})), y_j).$} \label{eq:mixup criterion}
\end{equation}
% \vspace{-8mm}
% \section{Combination between Manifolds}
% In this section, we investigate three possible variants for combining image space and latent space perturbations.

\section{Joint Space Threat Model}
% \Andy{Why is `combining` (1) and (2) important? Is it a more comprehensive / stronger threat model? Is there any apparent benefit if (1) and (2) are jointly solved?}

% \Andy{Or can you claim that (6) is a more general formulation for dual manifold learning, and (1) (2) are special cases? [Generalized DMAT]}
The ideal robust classifier should have the following properties: 1) Good standard accuracy; 2) Robust to seen attacks/known threat model; 3) Generalized to unseen attacks. AT provides satisfactory robustness to known threat models while sacrificing 1) and 3). To achieve 3), one can craft adversarial samples under a more comprehensive threat model. \cite{maini2020adversarial} considers the union of threat models by having average or maximum of the adversarial samples. However, it is not robust with adversarial attacks outside the union of threat model and computationally costly as it crafts each adversarial sample under the threat model within the union. \cite{lin2020dual} considers on-manifold adversarial samples, which improve standard accuracy and generalize to unseen attacks, but it requires exact manifold information. \cite{laidlaw2021perceptual} considers a comprehensive threat model called Neural Perceptual Threat Model (NPTM), which improves robustness to unseen attacks but sacrifices standard accuracy and requires additional relaxation in the attack algorithm. 

In light of these, we propose the Joint Space Threat Model (JSTM), which considers both image and latent space perturbations in one adversarial sample. To obtain a wider threat model, we consider the combination of image and latent space perturbations instead of considering the union of threat model and having a larger parameter space, i.e. $\Delta \cup \Lambda$. Mathematically, JSTM can be expressed as
\begin{equation}
    \max_{\delta \in \Delta, \lambda \in \Lambda} \cL (f_{\theta}(G(G^{-1}(x) + \lambda) + \delta), y_{\text{true}})
    \label{eq:double}
\end{equation}
\noindent where $G$ is Flow-based model. Since $G$ is invertible, exact manifold assumption holds. Moreover, JSTM is a special case of NPTM which does not require additional relaxations to craft adversarial attack.
% \vspace{-2mm}
\begin{lemma} \label{lemma}
    Assume image space perturbation $\delta=0$. Then JSTM is NPTM with the neural perceptual distance $\phi=G^{-1}$. 
\end{lemma}
\begin{proof}\renewcommand{\qedsymbol}{}
% \vspace{-2mm}
\setlength{\abovedisplayskip}{0pt}
\setlength{\belowdisplayskip}{0pt}
    \begin{align*}
    % \vspace{-2mm}
        &\max_{\hx} \cL (f_{\theta}(\hx), y_{\text{true}})~s.t.~\Vert G^{-1}(\hx) - G^{-1}(x) \Vert < \eta\\
        \iff~&\max_{\lambda} \cL (f_{\theta}(G(z+\lambda)), y_{\text{true}})~ \\ 
        &~~~~~~~~~~s.t.~\Vert G^{-1}(G(z+\lambda)) - G^{-1}(G(z)) \Vert < \eta\\
        \iff~&\max_{\lambda \in \Lambda} \cL (f_{\theta}(G(z+\lambda)), y_{\text{true}})
    \end{align*}    
    % \vspace{-4mm}
\end{proof}
\vspace{-3mm}
\noindent The second line holds as $G$ is invertible so for any $\hx$, there exists $\lambda$ such that $G(z+\lambda)=\hx$. 

\noindent To optimize \eqref{eq:double}, we propose the Joint Space Attack (JSA) algorithm, which uses the sign of gradients to update the $\lambda$ and $\delta$ similar to PGD attacks. Given an initial latent vector $z=G^{-1}(x)$, we have 
% \vspace{-1mm}
\begin{equation}
    \label{eq:full double algo}
    \begin{split}
    &\lambda_{k+1}= \eta_{iter} \cdot sign \left(\nabla_{\lambda_k} \cL(f_\theta (G(z + \lambda_k)+\delta_k), y_{\text{true}}) \right), \\ 
    &\delta_{k+1} = \epsilon_{iter} \cdot sign \left(\nabla_{\delta_k} \cL(f_\theta (G(z + \lambda_k)+\delta_k), y_{\text{true}}) \right), \\
    \end{split}
\end{equation}
\noindent where $\epsilon_{iter}$ and $\eta_{iter}$ are the attack step size at each iteration for image space and latent space, respectively. Also, there is a constraint on the JSA adversarial samples $\hx=\mathtt{clip}(G(z+\lambda)+\delta)$ within the image range, i.e. $\hx \in \cX$ and $\mathtt{clip}$ is the clip operator to clip the image within the range. Since we leverage the flow-based model, no additional relaxation is needed, while Perceptual Projected Gradient Descent (PPGD) under NPTM needs Taylor’s approximation.

\section{Optimized Perturbation with Interpolated Images} \label{sec: robust mixup}
% Since the DMT model creates a strong attack, training with the adversarial samples from DMT improves the robustness of the classifier but reduces the generalization performance. To mitigate this drawback while preserving certain robustness, interpolation-based training techniques \cite{zhang2018mixup,pang2019mixup,verma2019manifold,lamb2019interpolated} are adopted, which are shown to be effective and have demonstrated promising performance in adversarial robustness and generalization improvements. In our work, we focus on using Input Mixup \cite{zhang2018mixup} to combine the off-manifold and on-manifold adversarial samples and which is summarized next. Input Mixup draws two pairs of data sample from the dataset, $(x_i, y_i) \sim \bP_X$ and $(x_j, y_j) \sim \bP_X$, and takes the convex combination between them in the image space $\tx = \ga x_i + (1 - \ga)x_j$. $\ga \in (0, 1)$ is a random variable with Beta distribution. The interpolated image $\tx$ is passed into the classifier $\ft$ with minimizing the convex combination of the cross entropy loss,
% \begin{equation}
%     \cL^{mix}_\ga(\tx, \ft, \ga) = \ga \cL(\ft(\tx), y_i) + (1 - \ga)\cL(\ft(\tx), y_j). \label{eq:mixup criterion}
% \end{equation}

To achieve better standard accuracy, one can use data augmentation. Input mixup is one of the popular data augmentation methods, and it is easy to combine with AT. Interpolated Adversarial Training (IAT), proposed by \cite{lamb2019interpolated}, combines AT and interpolation-based training to design a robust classifier. A mixture of clean perturbed images is used in IAT. Although IAT demonstrates certain adversarial robustness, it is not optimized from a mathematical perspective in that the interpolated images are not guaranteed to maximize the cross-entropy loss. Suppose $x_i, x_j$ are two images, $y_i, y_j$ are the corresponding labels and $\mathtt{IM}(\vect{x}) = \ga x_i + (1-\ga) x_j$ be the interpolated image. Let $\hx^{img}_i$ and $\hx^{img}_j$ be the perturbed image with perturbation $\delta_i$ and $\delta_j$ respectively. The interpolated perturbation need not be the optimized perturbation w.r.t the interpolated image $\mathtt{IM}(\vect{x})$ as,
% \vspace{-3mm}
\begin{equation} \label{eq: normal mixup not good}
% \vspace{-2mm}
\setlength{\abovedisplayskip}{3pt}
\setlength{\belowdisplayskip}{3pt}
    \begin{split}
        \mathtt{IM}(\vect{\hx^{img}}) &= \ga \hx^{img}_i + (1-\ga) \hx^{img}_j \\
        &= \mathtt{IM}(\vect{x}) + [\ga \delta_i + (1-\ga) \delta_j] \\
        &\neq \mathtt{IM}(\vect{x}) + \argmax_{\delta \in \Delta} \cL^{mix}_\ga(\mathtt{IM}(\vect{x}) + \delta),
    \end{split}
\end{equation}
\begin{algorithm}[t]
\begin{algorithmic}[1]
\Require Training set $\cD_{tr}$, classifier $\ft$, generative model $G$, parameter of Beta distribution $\tau$.
\State Initialize $\ft$
\For{epoch = 1, \dots, k}
    \State Sample $\{ x_i, y_i, z_i \}, \{ x_j, y_j, z_j \} \sim \cD_{tr}$ 
    % \State $\hx^{lat} = {OM\_attack}(G(z_i), y_i)$ \Comment{\parbox[t]{.27\linewidth}{Run OM-PGD attack}}
    \State Sample $\ga \sim Beta(-\tau,\tau)$
    \State $\mathtt{IM}(\vect{x})  = \ga x_i + (1 - \ga) x_j$ \Comment{Input Mixup in $X$}
    % \State $\mathtt{IM}(\vect{z}) = \ga z_i + (1 - \ga)z_j$ \Comment{Input Mixup in $Z$}
    \State $\mathtt{IM}(\vect{y}) = \ga y_i + (1 - \ga) y_j$
    \State $\vect{z}^{mixup} = G^{-1}(\mathtt{IM}(\vect{x}))$
    % \State $\bx^{img}  = \mathtt{PGD}(\mathtt{IM}(\vect{x}), \mathtt{IM}(\vect{y}))$ \Comment{Run PGD attack}
    \State $\hx  = \mathtt{JSA}(\vect{z}^{mixup}, \mathtt{IM}(\vect{y}))$ \Comment{Run JSA}
    \State $\cL^{mix}_\ga(\hx)=\ga \cL(\ft(\hx, y_i)) + (1 - \ga)\cL(\ft(\hx, y_j))$
    \State $g \gets \nabla_{\theta}\cL^{mix}_\ga$ \Comment{Gradient of the mixup loss}
    \State $\theta \gets Step(\theta, g)$ \Comment{\parbox[t]{.35\linewidth}{Update parameters using gradients $g$}}
\EndFor
\end{algorithmic}
\caption{Interpolated Joint Space Adversarial Training (IJSAT)}
\label{alg:model1}

\end{algorithm}

% \begin{table}[t]
% \caption{Parameter settings for the novel attacks.}
% \label{table:unseen_attacks_param}
% \centering
% \small
% \begin{tabular}{lcccccc}
% \toprule
%   & Fog    &    Snow    & Elastic    &    Gabor   &     JPEG   &     $L_2$ \\
% \midrule
% $\epsilon$  & 128    &    0.062  & 0.500        &     12.500   &   1024     &   1200 \\
% $\epsilon_{iter}$   & 0.002 & 0.002 & 0.035 & 0.002 & 72.407 & 170 \\
% \bottomrule
% \end{tabular}
% \vspace{-4mm}
% \end{table}

% \vspace{-2mm}

In order to maximize the loss to create a strong attack on interpolated images that can fool the classifier, the perturbation step should be the final step to ensure that the resulting adversarial samples are optimized to fool the classifier. Therefore, we consider the following \textit{Robust Mixup strategy}. Suppose we have the interpolated images in image space $\mathtt{IM}(\vect{x})$ and the interpolated latent vectors $\mathtt{IM}(\vect{z})$. Then we apply the PGD attack on $\mathtt{IM}(\vect{x})$ and DPGD on $\mathtt{IM}(\vect{z})$. In other words, we use \eqref{eq:mixup criterion} to iteratively maximize the loss, i.e. 
\begin{equation} \label{eq: delta mixup}
\setlength{\abovedisplayskip}{6pt}
\setlength{\belowdisplayskip}{6pt}
    \begin{split}
        \delta_{k+1} &= \hat{\epsilon} \cdot sign \left(\nabla_{\delta_{k}} \cL^{mix}_\ga(\mathtt{IM}(\vect{x})_k + \delta_k) \right) \\
        &= \hat{\epsilon} \cdot sign \Big[\nabla_{\delta_{k}} \big(\ga \cL(\ft(\mathtt{IM}(\vect{x})_k + \delta_k), y_i) \\
        &+ (1 - \ga)\cL(\ft(\mathtt{IM}(\vect{x})_k + \delta_k), y_j)\big)\Big].
    \end{split}
\end{equation}
% \vspace{-2mm}
% \noindent Similarly, we can use \eqref{eq:mixup criterion} to find the latent space perturbation $\lambda$ iteratively. Hence, we have the following {\it Robust Interpolated Dual Manifold Adversarial Training (RIDMAT)} framework:
% \vspace{-2mm}
% \begin{equation} \label{eq: RIDMAT}
%     \boxed{\begin{split}
%         & \text{\bf RIDMAT Optimization:} \\
%         % &\quad \min_{\theta} \big\{ \cL^{mix}_\ga(\bx^{img}, \ft, \ga) + \cL^{mix}_\ga(\bx^{dual}, \ft, \ga) \big\} \label{eq: robust intra-perturbation mixup} \\
%         & \min_{\theta} \Big\{  \max_{\delta_1 \in \Delta} \cL^{mix}_\ga(\mathtt{IM}(\vect{x}) + \delta_1) \\
%         &+ \max_{\delta_2 \in \Delta, \lambda \in \Lambda} \cL^{mix}_\ga \big( G(\mathtt{IM}(\vect{z}) + \lambda) + \delta_2 \big) \Big\}
%     \end{split}}
% \end{equation}
% \noindent where $\tx = \ga x_i + (1 - \ga) x_j$ is the interpolated image and $\tz = \ga z_i + (1 - \ga)z_j$ is the interpolated latent vector.

\noindent The interpolated step followed by perturbation will ensure that the resultant images are optimized and generate a stronger attack.
\section{Interpolated Joint Space Adversarial Training}
Joint Space Attack can be used to harden a classifier against both seen and unseen attacks. The intuition, which we verify in Section \ref{sec: Experiment}, is that on-manifold adversarial samples improve the generalization of the classifier given exact manifold assumption holds according to \cite{lin2020dual}. JSA also lies within NPTM \cite{laidlaw2021perceptual} which has been shown to be a comprehensive threat model and provide good robustness to unseen attacks. On the other hand, the image space perturbation in JSA can help the classifier to defend $L_p$ attacks such as FGSM, PGD, and Auto Attacks. Combined with the proposed Robust Mixup strategy, we have the following {\it Interpolated Joint Space Adversarial Training (IJSAT)} framework:
% \vspace{-2mm}
\begin{equation} \label{eq: RIDMAT}
    \resizebox{0.9\hsize}{!}{\boxed{\begin{split}
        & \text{\bf IJSAT Optimization:} \\
        % &\quad \min_{\theta} \big\{ \cL^{mix}_\ga(\bx^{img}, \ft, \ga) + \cL^{mix}_\ga(\bx^{dual}, \ft, \ga) \big\} \label{eq: robust intra-perturbation mixup} \\
        % & \min_{\theta} \Big\{  \max_{\delta_1 \in \Delta} \cL^{mix}_\ga(\mathtt{IM}(\vect{x}) + \delta_1) \\
        & \min_{\theta} \bigg\{ \max_{\delta \in \Delta, \lambda \in \Lambda} \cL^{mix}_\ga \Big( G \big(G^{-1}(\mathtt{IM}(x)) + \lambda \big) + \delta \Big) \bigg\}
    \end{split}}}
\end{equation}
\noindent The detailed algorithm is described in Algorithm Block \ref{alg:model1}. 
\section{Experiments} \label{sec: Experiment}

\subsection{Implementation Details}

\noindent \textbf{Dataset:} We evaluate the proposed method on the CIFAR-10, CIFAR-10-C \cite{hendrycks2019robustness}, CIFAR-100 and OM-ImageNet \cite{lin2020dual} datasets. In the OM-ImageNet dataset, all images are from ImageNet and projected by StyleGAN. In other words, all images are on-manifold and have the corresponding latent vectors. The CIFAR-100 results are shown in supplementary material. 

% \hspace{1mm}

% \noindent \textbf{Models:} For CIFAR-10 and CIFAR-100, we use ResNet-18. We follow \cite{zhang2019theoretically} to have batch size $128$ with 120 epochs. We use the SGD optimizer with setting initial learning rate to 0.1, momentum to 0.9 and weight decay to $2\times 10^{-4}$ .The learning rate drops by 0.1 at the 75-th, 90-th and 100-th epochs. We set $\epsilon=8/255$, $\epsilon_{iter}=2/255$ for image space attacks and $\eta=0.02$, $\eta_{iter}=0.005$ for latent space attacks with $10$ iteration steps. For OM-ImageNet, we follow \cite{lin2020dual} to use ResNet-50 \cite{he2016deep} and train the classifier with 20 epochs. We use the SGD optimizer with the cyclic learning rate scheduling strategy in~\cite{Wong2020Fast}, momentum $0.9$, and weight decay $5\times 10^{-4}$. We set $\epsilon=4/255$, $\epsilon_{iter}=1/255$ for image space attacks and $\eta=0.02$, $\eta_{iter}=0.005$ for latent space attacks with $5$ iteration steps. For input mixup, we use $\tau=0.1$ for the random scalar, i.e. $\ga \sim \text{Beta}(-\tau, \tau)$. For TRADES and MART, we use $\beta=6.0$ for the KL-divergence loss. Different training setting will be explicitly mentioned otherwise. For all models trained by the proposed method, we use GLOW \cite{DBLP:conf/nips/KingmaD18} as the generator.

\noindent \textbf{Models:} For CIFAR-10 and CIFAR-100, we use ResNet-18. We follow \cite{zhang2019theoretically} to have batch size $128$ with 120 epochs. We use the SGD optimizer with setting initial learning rate to 0.1, momentum to 0.9 and weight decay to $2\times 10^{-4}$ .The learning rate drops by 0.1 at the 75-th, 90-th and 100-th epochs. We set $\epsilon=8/255$, $\epsilon_{iter}=2/255$ for image space attacks and $\eta=0.02$, $\eta_{iter}=0.005$ for latent space attacks with $10$ iteration steps. For OM-ImageNet, we follow \cite{lin2020dual} to use ResNet-50 \cite{he2016deep} and train the classifier with 20 epochs. We use the SGD optimizer with the cyclic learning rate scheduling strategy in~\cite{Wong2020Fast}, momentum $0.9$, and weight decay $5\times 10^{-4}$. We set $\epsilon=4/255$, $\epsilon_{iter}=1/255$ for image space attacks and $\eta=0.02$, $\eta_{iter}=0.005$ for latent space attacks with $5$ iteration steps. For input mixup, we use $\tau=0.1$ for the random scalar, i.e. $\ga \sim \text{Beta}(-\tau, \tau)$. For TRADES and MART, we use $\beta=6.0$ for the KL-divergence loss. Different training setting will be explicitly mentioned otherwise. For all models trained by the proposed method, we use GLOW \cite{DBLP:conf/nips/KingmaD18} as the generator.

\vspace{1mm}
% \blue{maybe a new section here like "Evaluation Settings"}
% \noindent \textbf{Evaluation Settings:}
\subsection{Evaluation Settings}
We evaluate our method in three aspects: {(I) Standard Accuracy}; {(II) Robustness} and {(III) Generalization}. 

\vspace{1mm}

\noindent \textbf{Standard Accuracy:} We compare our model with \textbf{Normal Training}, \textbf{VAE-GAN} \cite{Stutz_2019_CVPR}–on-manifold adversarial samples using VAE-GAN, \textbf{Cutout} \cite{devries2017cutout}-data augmentation with input masking, \textbf{Mixup} \cite{zhang2018mixup}-data augmentation with interpolated images, \textbf{Randomized-LA} and \textbf{Adversarial-LA} \cite{yuksel2021semantic}-on-manifold adversarial samples using GLOW in CIFAR-10 dataset.

% \hspace{1mm}
\vspace{1mm}

\noindent \textbf{Robustness:} We compare our model with \textbf{$L_\infty$ AT} \cite{madry2017towards}, \textbf{DMAT} \cite{lin2020dual}-on-manifold adversarial training using StyleGAN, \textbf{IAT} \cite{lamb2019interpolated}-interpolated adversarial training using Input Mixup, \textbf{TRADES} \cite{zhang2019theoretically}-adversarial training using regularized surrogate loss to encourage smooth output and \textbf{MART} \cite{wang2019improving}-adversarial training that explicitly differentiates the misclassified and correctly classified examples in CIFAR-10, CIFAR-100 and OM-ImageNet dataset. For CIFAR-10 and CIFAR-100, we use PGD-20, Auto-Attack \cite{croce2020reliable}, which is an ensemble of four diverse attacks and the proposed JSA attack as the seen attacks. For unseen attacks, we use Elastic, JPEG and $L_2$ from \cite{kang2019robustness}. For OM-ImageNet, we follow \cite{lin2020dual} to have a total of eleven attacks. 

\vspace{1mm}

\noindent \textbf{Generalization:} We compare our model with \textbf{$L_\infty$ AT}, \textbf{$L_2$ AT}, \textbf{Fast PAT} \cite{laidlaw2021perceptual}-adversarial training using perceptual adversarial samples, \textbf{AdvProp} \cite{xie2020adversarial}-uses separate auxiliary batch norm for adversarial samples and \textbf{RLAT} \cite{kireev2021effectiveness}-relaxation of the LPIPS adversarial training in CIFAR-10-C dataset \cite{hendrycks2019robustness}, which is a CIFAR-10 dataset with a total of fifteen common image corruptions.

\begin{table}[t!]
\setlength\tabcolsep{3pt}
\caption{Classification accuracy against various attacks applied to CIFAR10 dataset. {Bold} values indicate the best performance.} \label{table:cifar10}
\centering
\small
\scalebox{0.80}{\begin{tabular}{@{}lr|rrr|rrr|r@{}}
\toprule
Method                &  Standard  &  $\text{PGD}^{20}$ & AA  & $\text{JSA}^{50}$  &  Elastic   &  JPEG &  $L_2$ & Avg\\
\midrule
Normal Training       & 94.69& 0.00 & 0.00 & 0.00 & 29.61 & 0.00 & 0.00 & 17.76\\
AT [PGD-5] \cite{madry2017towards}  & 84.15  & 49.85 & \underline{44.71} & 45.71& 45.16  & 26.71 & 18.75 & 45.01\\
DMAT \cite{lin2020dual}   & 82.77&	45.01&	34.08&	36.78&	56.72&	\textbf{38.42}&	\textbf{28.6}&	46.05\\
IAT \cite{lamb2019interpolated}  & \textbf{86.45}	&47.88&	41.27&	40.17&	59.09&	31.15&	22.97&	47.00 \\
TRADES \cite{zhang2019theoretically}   & 82.86&	\textbf{53.88}&	\textbf{48.87}&	\textbf{48.27}&	54.73&	32.55&	22.72 & 49.13\\
MART \cite{wang2019improving}  & 82.81&	53.25&	45.58&	46.73&	56.48&	28.4&	20.4 & 47.67\\
IJSAT \textbf{(ours)} & 83.16&	53.68&	47.58&	47.43&	\textbf{59.74}&	30.01&	23.18
 & \textbf{49.25} \\
\bottomrule
\end{tabular}}
% \vspace{-2mm}
\end{table}

\subsection{Main Results}

\noindent \textbf{CIFAR-10 Robustness:} We show the robustness results on CIFAR-10 in Table~\ref{table:cifar10}. First, by comparing the baseline AT method and our proposed method IJAST, we observe significant improvement in robustness against all six attacks. For DMAT, since it uses a GAN model to project the image to obtain the latent vectors, which cannot reconstruct the images exactly, robustness against off-manifold attacks drops as it does not generalize well without exact manifold assumption. Surprisingly, DMAT gains robustness to unseen attacks because of inexact image reconstruction (See Fig.~\ref{fig:overview}(c)). IAT has the best standard accuracy as it uses Input Mixup and trains with clean images. However, we observe robustness against off-manifold attacks drops as there is a trade-off between standard accuracy and robustness. The proposed IJSAT has a comparable robustness against PGD-20 attacks with TRADES and MART while TRADEs achieves the best performance in Auto Attack and IJSAT achieves the second best. Unlike TRADES, which uses a surrogate function, and MART, which uses the boosted cross-entropy function to gain robustness, IJSAT only uses adversarial samples crafted from the proposed JSTM to achieve similar robustness gain while achieving the best standard accuracy. TRADES and MART can be further improved with IJSAT as shown in Sec. \ref{sec: Applying JSA in Other Adversarial Training Methods}. Overall, IJSAT achieves the best average performance.

% \vspace{1mm}

% \noindent \textbf{CIFAR-100 Robustness:} To test our method on a dataset with more classes, we evaluate our methods on CIFAR-100. The results are in Table~\ref{table:cifar100}. We observe similar results as in CIFAR-10: 1) IJSAT has the best overall results; 2) MART, TRADES, and IJSAT achieve comparable $L_\infty$ robustness; 3) IJSAT achieves the best unseen attack robustness. 

\vspace{1mm}

\begin{table}[t!]
\setlength\tabcolsep{5pt}
\caption{Left: Standard accuracy on CIFAR-10. Right: Common corruption accuracy on CIFAR-10-C. {Bold} values indicate the best performance in that section.} \label{table:cifar-10 standard and corruption}
\centering
\small
\scalebox{1.}{\begin{tabular}{@{}lr|lr@{}}
\toprule
Method                &  Standard   &  Method & Corrupted \\
\midrule
Normal Training &   95.2 &Normal Training  &   74.3 \\
VAE-GAN \cite{Stutz_2019_CVPR}   &   94.2 & $L_\infty$ AT    &   82.7 \\
Cutout \cite{devries2017cutout}  &   96.0&  $L_2$ AT   &   83.4 \\
Input Mixup \cite{zhang2018mixup}  &   95.9 & Fast PAT \cite{laidlaw2021perceptual}   &   82.4 \\
Randomized-LA \cite{yuksel2021semantic} & 96.3& AdvProp  \cite{xie2020adversarial}  &   82.9 \\
Adversarial-LA \cite{yuksel2021semantic}  &   96.6&  RLAT \cite{kireev2021effectiveness}  &   84.1 \\
IJSAT \textbf{(ours)}  &   \textbf{96.9}&  IJSAT \textbf{(ours)}   &   \textbf{84.6} \\
\bottomrule
\end{tabular}}
% \vspace{-5mm}
\end{table}

\begin{table*}[t!]
\caption{Classification accuracy against different attacks applied to OM-ImageNet test set. {Bold} values indicate the best performance in that section.} \label{table:OM_imagenet all}
\centering
\small
\scalebox{1.}{\begin{tabular}{@{}lr|rrr|r|rrrrrr@{}}
\toprule
Method                &  Standard  & FGSM  &  PGD-50 & MIA & OM-PGD-50 &  Fog     &  Snow    &  Elastic   &   Gabor   &  JPEG &  $L_2$ \\
\midrule
Normal Training       &   74.72&  2.59  &  0.00&  0.00&  0.26   &  0.03  &  0.06 &  1.20   &   0.03 &  0.00 & 1.7 \\
AT [PGD-5] \cite{madry2017towards}            &       73.31  &  48.02   &  38.88&   39.21  &  7.23  &  19.76 & 46.39 &  50.32  &  50.43 & 10.23 & 41.98\\
DMAT \cite{lin2020dual}    &    \textbf{77.96}   &  49.12   & 37.86 &   37.65    & 20.53  & 31.78& 51.19 &  56.09  &  51.61  & 14.31  & 51.36 \\
IAT \cite{lamb2019interpolated}  &       76.75  &  48.33   &  37.58 &   38.05    &  9.41  &  27.07 & 48.36 &  52.51  &  51.02 & 13.33 & 43.24\\
% TRADES \cite{zhang2019theoretically}  &      74.61  &  47.38   &  39.01 &   39.39    &  7.73 & 1.92 &  24.33 & 47.93 &  53.10  &  51.30 & 18.53 & 42.79\\
TRADES \cite{zhang2019theoretically}  &      72.34  &  53.29  &  47.76 &   47.84    &  10.04  &  26.86 & 51.13 & 55.57  &  55.79 & 10.28 & 46.75\\
MART \cite{wang2019improving}  &      72.86  &  52.28   &  45.43 &   45.62    &  8.61 &  25.34 & 51.00 &  54.40  &  54.27 & 8.95 & 44.63\\
% AT [JSA-5] \textbf{(ours)}    &  69.04&  53.47&  49.24&   49.25& {31.20} & 16.02& 31.68& 52.53& 56.11 & 55.97& 22.39& 55.50 \\
IJSAT \textbf{(ours)} &  73.72&  \textbf{56.55} &  \textbf{50.85}&   \textbf{51.07} & \textbf{23.92}  &  \textbf{32.7} & \textbf{57.23}& \textbf{59.95}  & \textbf{59.82} & \textbf{22.47} & \textbf{56.40}  \\
\bottomrule
\end{tabular}}
\vspace{-4mm}
\end{table*}

\noindent \textbf{OM-ImageNet Robustness:} To test our method with high-resolution images, we evaluate our methods with OM-ImageNet. Since every image has the corresponding latent vectors in the OM-ImageNet dataset, we do not need to use the Flow-based model to compute the latent vectors to ensure exact manifold assumption. The results are in Table~\ref{table:OM_imagenet all}. For off-manifold attacks, TRADES, MART, and IJSAT generally achieve better performance than others. IJSAT performs the best as the JSA adversarial samples provide significant robustness to off-manifold attacks. For on-manifold robustness, since only DMAT and IJSAT consider on-manifold adversarial samples, they have significantly better on-manifold robustness than others. Among these two models, IJSAT has better on-manifold robustness than DMAT. For novel attacks, DMAT and IJSAT achieve generally better performance than the others. This is consistent with \cite{lin2020dual} that on-manifold adversarial samples improve generalization to novel attacks. Also, the mixup strategy improves generalization. From Table~\ref{table:OM_imagenet all}, IAT is consistently better than AT [PGD-5] in terms of robustness to novel attacks. Therefore, the proposed Robust Mixup strategy also boosts generalization performance, and hence IJSAT achieves the best robustness to novel attacks.

\vspace{1mm}

\noindent \textbf{Standard Accuracy:} We compare the proposed method to other data augmentation methods, and the results are shown in Table~\ref{table:cifar-10 standard and corruption}. Since image space perturbation decreases standard accuracy, we train a model by IJSAT with no image space perturbation, and other hyperparameters follow \cite{yuksel2021semantic}. This can be used as data augmentation. Note that although both Adversarial-LA and the proposed method use GLOW to craft on-manifold adversarial samples, Adversarial-LA uses $L_2$ norm in the latent space and does not have a mixup strategy. In contrast, the proposed IJSAT uses $L_\infty$ norm in the latent space with the proposed Robust Mixup strategy. Therefore, IJSAT achieves the best standard accuracy. 

\vspace{1mm}

\noindent \textbf{Generalization:} We use CIFAR-10-C to demonstrate the generalization power of IJSAT. To achieve a good performance in CIFAR-10-C, we train our model with image space budget $\epsilon=1/255$ and $\epsilon_{iter}=\epsilon/4$ and other hyperparameters follow \cite{kireev2021effectiveness}. We show the results in Table~\ref{table:cifar-10 standard and corruption}. Both Fast PAT and RLAT lie within NPTM and use perceptual adversarial samples to train the model, and hence they generalize well to corrupted images. From Lemma \ref{lemma}, we know that the latent space perturbation in JSA samples lies within NPTM. On the other hand, JSA samples also include image space perturbation, and hence IJSAT achieves the best result. 
\subsection{Ablation Studies}
Since IJSAT has multiple components, we demonstrate the improvements using the CIFAR-10 dataset, one component at a time. The results are shown in Table~\ref{table:cifar10 ablation}.
% \vspace{-4mm}

\vspace{-4mm}

\subsubsection{Importance of Having Exact Manifold Information}
To demonstrate the importance of having exact manifold information, we compare the on-manifold adversarial image crafted from GAN and flow-based models. We train a GAN model with CIFAR-10 and use it to craft JSA samples. Then we train a model with them, denoted as AT [$\text{JSA}^{10}$-GAN]. In other words, the only difference between AT [$\text{JSA}^{10}$-GAN] and AT [$\text{JSA}^{10}$] is that the former uses a GAN as the generator while the latter uses the Flow-based model. In Table~\ref{table:cifar10 ablation}, we observe a significant improvement with standard accuracy ($\sim 10\% \uparrow$), and robustness of $L_\infty$ attacks ($\sim 14\% \uparrow$) and unseen attacks ($\sim 9\% \uparrow$) for AT [$\text{JSA}^{10}$]. As $\text{JSA}^{10}$ samples have the exact manifold information, the on-manifold adversarial samples preserve the details of the images while $\text{JSA}^{10}$-GAN does not. As shown in Fig.~\ref{fig:overview}(c), we observe a large difference between the original image and the projected image. Even though the projected image has similar semantic details, the projected image makes the classifier hard to generalize well when natural images are evaluated.

\vspace{-4mm}

\subsubsection{JSA Adversarial Samples Improves Robustness Significantly}
To demonstrate the robustness gain whilst training with JSA samples, we train an AT model with JSA samples and compare them with standard AT, denoted as AT [$\text{PGD}^{10}$] and AT [$\text{JSA}^{10}$] respectively. From Table~\ref{table:cifar10 ablation}, we observe significant improvement with robustness of $L_\infty$ attacks ($\sim 3\% \uparrow$) and unseen attacks ($\sim 7\% \uparrow$) AT [$\text{JSA}^{10}$]. Since the only difference between these two models is the adversarial sample, this indicates that JSA samples provide more robustness to the trained model than $L_\infty$ samples.

\begin{table}[t!]
\setlength\tabcolsep{3pt}
\caption{Ablations studies of IJSAT. Classification accuracy against different attacks applied to CIFAR-10. {Bold} values indicate the best performance in that section.} \label{table:cifar10 ablation}
\centering
\small
\scalebox{0.8}{\begin{tabular}{@{}lr|rrr|rrr|r@{}}
\toprule
Method                &  Standard   &  $\text{PGD}^{20}$ & AA  & $\text{JSA}^{50}$   &  Elastic    &  JPEG &  $L_2$ & Avg\\
\midrule
Normal Training &   \textbf{{94.69}} &  0   &  0 &  0   &  29.61  &  0   &  0  &  17.76 \\
AT [$\text{PGD}^{10}$]             &       {84.15} &  49.85   &  44.71 &  45.71  &  45.16  &  26.76   &  18.75 &  45.01 \\
~~$\hookrightarrow \epsilon=16/255$    &       68.91	&52.78&	47.32& \textbf{47.95} &	47.09&	\textbf{32.97}&	\textbf{27.02}&	46.29
 \\
~~$\hookrightarrow \epsilon=32/255$    &       37.11&	31.67&	28.91& 30.14	&	21.67&	10.06&	11.06&	14.26\\
AT [$\text{JSA}^{10}$-GAN]  &       {76.43} &  47.04   &  43.10 &  43.82  &  39.35  &  18.63   &  {26.09} &  42.07 \\
AT [$\text{JSA}^{10}$]     &  {82.75} &  53.15   &  47.47 &  47.10   &  59.67  &  28.79   &  22.69  &  48.80 \\
~~+ mixup  &  {83.08} &  53.12   &  \textbf{47.59} &  47.04  &  59.65  &  28.56   &  {23.56}  &  48.94 \\
IJSAT &  {83.16} &  \textbf{53.68}   &  47.58 &  {47.43}  &  \textbf{59.74}  &  {30.01}   &  23.18  &  \textbf{49.25} \\
\bottomrule
\end{tabular}}
% \vspace{-3mm}
\end{table}

Since the JST model considers image and latent space perturbations, the resultant perturbations will have a larger attack budget. To investigate whether the success of IJAST is due to adversarial training with a large attack budget, we conduct experiments with models trained with larger attack budgets. We train an AT model with image space budget $\epsilon=16/255$ and $\epsilon=32/255$ and compare their robustness against both $L_\infty$ and unseen attacks. We compare them with the proposed AT [$\text{JSA}^{10}$], using $L_\infty=8/255$ as the attack budget in image space and $L_\infty=0.002$ in latent space. From Table~\ref{table:cifar10 ablation}, we observe that the standard accuracy drops when the models are trained with a larger $L_\infty$ budget. The AT [$\text{PGD}^{10}$, 16/255] has similar $L_\infty$ robustness with AT [$\text{JSA}^{10}$], which AT [$\text{JSA}^{10}$] is slightly better. However, this model has a $14\%$ drop in standard accuracy compared with AT [$\text{JSA}^{10}$]. It is even worse when we increase the budget to $32/255$. This experiment shows that the robustness gain of JSA samples does not merely rely on larger attack budgets.

\vspace{-4mm}
% \vspace{-2mm}
\subsubsection{Robust Mixup Further Boosts Standard Accuracy, Robustness and Generalization}
Mixup as a data augmentation method usually helps in classifier training. However, crafting the adversarial samples and then doing the mixup will negatively impact the robustness. We denote this model as AT [$\text{JSA}^{10}$]-mixup. From Table~\ref{table:cifar10 ablation}, using mixup in AT [$\text{JSA}^{10}$] cannot guarantee improving robustness. This is because the interpolated perturbations need not be the optimized perturbations to the interpolated images (See Eq. \eqref{eq: normal mixup not good}). When applying Robust Mixup (denoted as IJSAT), both standard accuracy and robustness are further improved. The perturbation step as the final step ensures that the adversarial sample maximizes the cross-entropy loss and fools the classifier. Training with these strong adversarial samples results in improving the robustness.

We plot the robust accuracy during training in Fig.~\ref{fig:cifar10 plot}. We can observe that IJSAT achieves the best robustness. Moreover, without mixup, the robustness of the model would decay after learning rate changes (at $75^{th}$ epoch). Input Mixup can help slightly (green line), while the proposed Robust Mixup can further improve robustness after learning rate changes (blue line), which reduces robust-overfitting.

\begin{figure}[t]
% \vspace{-2mm}
\setlength\tabcolsep{0pt}
\renewcommand{\arraystretch}{0.5}
\scalebox{0.95}{\begin{tabular}{ccccc}
\begin{subfigure}[t]{0.1\textwidth}
\includegraphics[width=\textwidth]{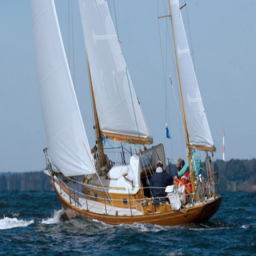}
\end{subfigure} &
% \hspace{-2mm}
\begin{subfigure}[t]{0.1\textwidth}
\includegraphics[width=\textwidth]{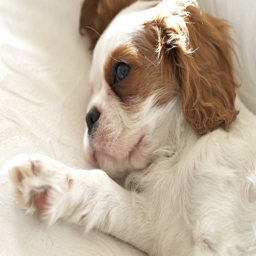}
\end{subfigure} &
% \hspace{-2mm}
% \\ \vspace{-1mm}
\begin{subfigure}[t]{0.1\textwidth}
\includegraphics[width=\textwidth]{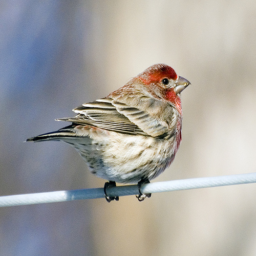}
\end{subfigure} &
% \hspace{-2mm}
% \hspace{-1.8mm}
\begin{subfigure}[t]{0.1\textwidth}
\includegraphics[width=\textwidth]{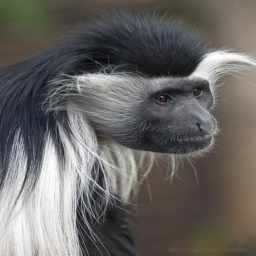}
\end{subfigure} &
% \hspace{-2mm}
\begin{subfigure}[t]{0.1\textwidth}
\includegraphics[width=\textwidth]{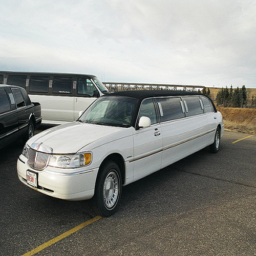}
\end{subfigure}
\\
\begin{subfigure}[t]{0.1\textwidth}
\includegraphics[width=\textwidth]{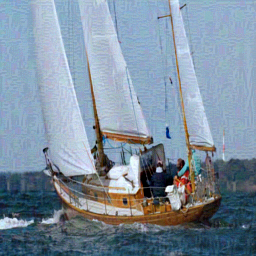}
\end{subfigure} &
% \hspace{-2mm}
\begin{subfigure}[t]{0.1\textwidth}
\includegraphics[width=\textwidth]{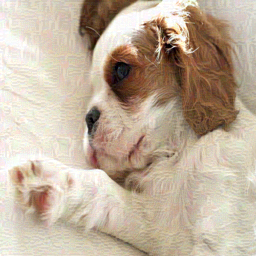}
\end{subfigure} &
% \hspace{-2mm}
% \\ \vspace{-1mm}
\begin{subfigure}[t]{0.1\textwidth}
\includegraphics[width=\textwidth]{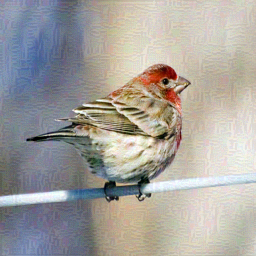}
\end{subfigure} &
% \hspace{-2mm}
% \hspace{-1.8mm}
\begin{subfigure}[t]{0.1\textwidth}
\includegraphics[width=\textwidth]{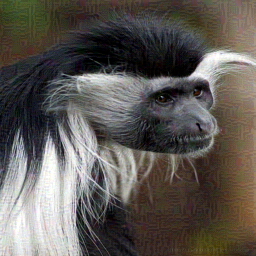}
\end{subfigure} &
% \hspace{-2mm}
\begin{subfigure}[t]{0.1\textwidth}
\includegraphics[width=\textwidth]{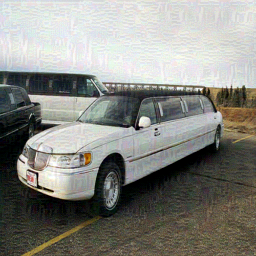}
\end{subfigure}
\\
\begin{subfigure}[t]{0.1\textwidth}
\includegraphics[width=\textwidth]{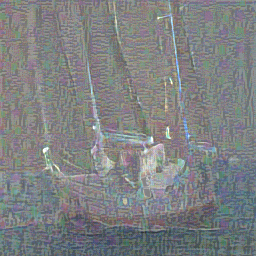}
\end{subfigure} &
% \hspace{-2mm}
\begin{subfigure}[t]{0.1\textwidth}
\includegraphics[width=\textwidth]{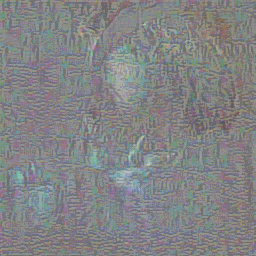}
\end{subfigure} &
% \hspace{-2mm}
% \\ \vspace{-1mm}
\begin{subfigure}[t]{0.1\textwidth}
\includegraphics[width=\textwidth]{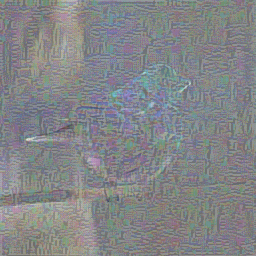}
\end{subfigure} &
% \hspace{-2mm}
% \hspace{-1.8mm}
\begin{subfigure}[t]{0.1\textwidth}
\includegraphics[width=\textwidth]{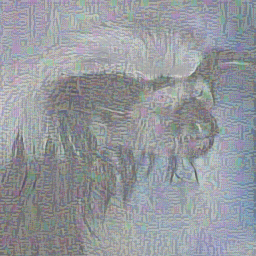}
\end{subfigure} &
% \hspace{-2mm}
\begin{subfigure}[t]{0.1\textwidth}
\includegraphics[width=\textwidth]{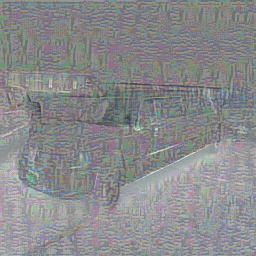}
\end{subfigure}
\end{tabular}}
% \vspace{-1mm}
% \vspace{-3mm}
\caption{Visualization of adversarial samples from JSA. Top: Original. Middle: JSA. Bottom: Magnified difference.}
\label{fig: JSA attack samples}
\vspace{-2mm}
\end{figure}
\subsection{Joint Space Attacks}
% \vspace{-2mm}
\begin{table}[t]
% \vspace{-2mm}
\setlength\tabcolsep{3pt}
\caption{Classification accuracy against different attacks applied to CIFAR-10 dataset.}
\label{table:cifar10 improvement}
\centering
\small
\scalebox{0.85}{\begin{tabular}{@{}lr|rr|rrr|r@{}}
\toprule
Method          &  Standard  &  $\text{PGD}^{20}$ &  AA       &  Elastic & JPEG  &   $L_2$   & Avg\\
\midrule
% Normal Training  &    94.69   & 0  & 0      & 29.61 &  0   &  0  & 24.02 \\
% Normal Training + IA  &    95.35   & 0.03  & 0     & 44.83 &  0.19   &  0.02  & 28.61 \\
% Difference  &    $\uparrow$0.66   & $\uparrow$0.03  & 0   & $\uparrow$15.22 &  $\uparrow$0.19   &  $\uparrow$0.02  & $\uparrow$4.59 \\
% \midrule
AT \cite{madry2017towards}  &    84.15   & 49.85  & 44.71      & 45.16 &  26.76   &  18.75  & 49.76 \\
AT \cite{madry2017towards} + JSA  &    {82.75} &  53.15   &  47.47   &  59.67  &  28.79   &  22.69  &  48.80 \\
Difference  &    $\downarrow$1.4   & $\uparrow$3.3  & $\uparrow$2.76   & $\uparrow$14.51 &  $\uparrow$2.03   &  $\uparrow$3.94  & $\uparrow$3.79 \\
\midrule
TRADES \cite{zhang2019theoretically}  &    82.86   & 53.88  & 48.87      & 54.73 &  32.55   &  23.72  & 53.79 \\
TRADES \cite{zhang2019theoretically} + JSA  &    82.70   & 54.22  & 49.10      & 58.66 &  32.55   &  25.18  & 54.68 \\
Difference  &   $\downarrow$0.16    &  $\uparrow$0.34 &  $\uparrow$0.23      & $\uparrow$3.93 &  0   &    $\uparrow$1.46&$\uparrow$0.89  \\
\midrule
MART \cite{wang2019improving}  &    82.81   & 53.25  & 45.58     & 56.48 &  28.4   &  20.47  & 52.36 \\
MART \cite{wang2019improving} + JSA  &    81.33   & 54.28  & 46.53      & 59.15 &  30.3   &  24.28  & 53.64 \\
Difference  &   $\downarrow$1.48    & $\uparrow$1.03  &   $\uparrow$0.95    &  $\uparrow$2.67&    $\uparrow$1.9 &$\uparrow$3.81    &$\uparrow$1.28  \\
%STANDARD  & 82.12\%  &  53.11\%  & 33.05\%  & 63.00\%  & 61.67\%  & 61.12\%  & 3.82\% & 43.97\% \\
\bottomrule
\end{tabular}}
% \vspace{-1mm}
% \vspace{-3mm}
\end{table}

In Table~\ref{table:cifar10 ablation}, we show the accuracy of different models against JSA. JSA is a strong attack, and it achieves similar results as AA. TRADES achieves the best while the proposed IJAST is the second best. Visualization results of JSA samples are shown in Fig.~\ref{fig:overview}. For the PGD attack, the perturbations are noise covering the entire image. For OM-PGD (GAN), the projection step in GAN leads to a significant difference between the adversarial image and the original image. For OM-PGD (Flow), since the Flow-based model crafts the perturbation, it preserves the semantic information but cannot provide robustness to $L_p$ attacks. For JSA, we observe that both image and latent space perturbations provide robustness to $L_p$ and unseen attacks. We empirically show the JSA perturbation would not change the semantic meaning of the images (See Fig.~\ref{fig: JSA attack samples}).

\begin{figure}[t!]
    \centering
    \captionsetup[subfigure]{justification=centering}
    % \begin{subfigure}[t]{0.25\textwidth}
    %     \raisebox{-6ex}{\rotatebox[origin=c]{90}{\text{Normal}}}
    % \end{subfigure}
    % \hspace{-3.3cm}
    % \begin{subfigure}[t]{0.2\textwidth}
    %     \raisebox{-\height}{\includegraphics[width=\textwidth]{images/supp/layer/batch_5074_inputs.png}}
    %     \caption*{Original}
    % \end{subfigure}
    \begin{subfigure}[t]{0.16\textwidth}
        \raisebox{-\height}{\includegraphics[width=0.48\textwidth]{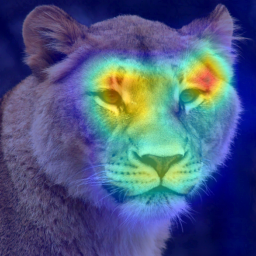}} \hspace{-2mm}
        \raisebox{-\height}{\includegraphics[width=0.48\textwidth]{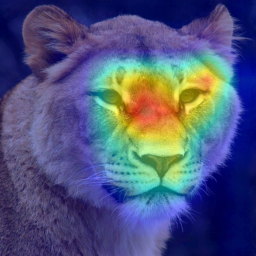}} 
        \caption*{AT}
    \end{subfigure}
    \hspace{-2mm}
    \hfill
    \begin{subfigure}[t]{0.16\textwidth}
        \raisebox{-\height}{\includegraphics[width=0.48\textwidth]{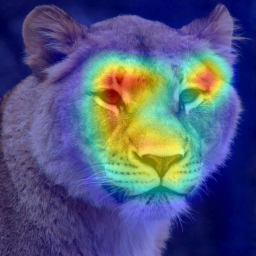}} \hspace{-2mm}
        \raisebox{-\height}{\includegraphics[width=0.48\textwidth]{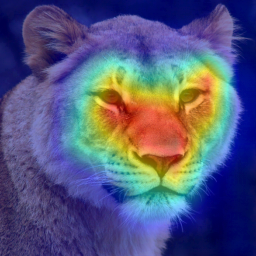}} 
        \caption*{TRADES}
    \end{subfigure}
    \hspace{-2mm}
    \hfill
    \begin{subfigure}[t]{0.16\textwidth}
        \raisebox{-\height}{\includegraphics[width=0.48\textwidth]{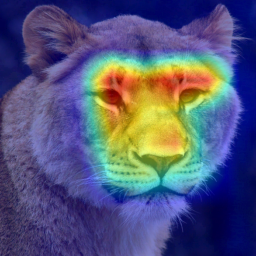}} \hspace{-2mm}
        \raisebox{-\height}{\includegraphics[width=0.48\textwidth]{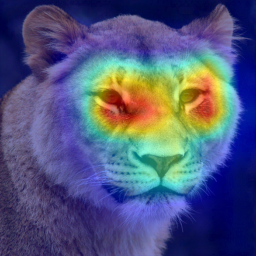}} 
        \caption*{MART}
    \end{subfigure}
    \caption{Visualization of CAM for different methods with and without JSA against PGD attack. {Left}: Original AT method. {Right}: AT method with JSA.} \label{fig:cam}
    \vspace{-2mm}
\end{figure}

% \subsection{Comparison with Models with Larger Architecture}
\subsection{Applying JSA in Other Adversarial Training Methods} \label{sec: Applying JSA in Other Adversarial Training Methods}
JSA also can be combined with existing AT methods. The results for CIFAR-10 dataset are shown in Table~\ref{table:cifar10 improvement}. For AT, TRADES, and MART, the robustness for $L_\infty$, $L_2$, on-manifold, and non-$L_p$ attacks is improved after applying JSA. This demonstrates the flexibility of applying JSA adversarial to existing adversarial training methods and enhance robustness. To understand how JSA improves robustness, we can use a class activation map (CAM) to obtain a good visual explanation. In this paper, we use GRAD-CAM \cite{selvaraju2017grad} to visualize. The CAM results are shown in Fig.~\ref{fig:cam}. For the existing adversarial training method, the heatmap has a higher confidence score and overlaps more with the semantic meaningful region (the face of the leopard) when JSA is used. 
\section{Conclusion and Discussion}
In this paper, we propose a novel threat model, JSTM, which considers both image and latent space perturbations. Under this threat model, we propose JSA and use it to train the classifier. This overcomes the drawbacks of AT not being able to generalize well to unseen attacks and lack of robustness of on-manifold adversarial training to $L_p$ attacks. To ensure JSA samples generalize well to real datasets, we exploit the invertibility of the Flow-based model to make the exact manifold assumption holds. To further improve robustness and prevent overfitting, we propose the Robust Mixup strategy. Our extensive experiments show that IJSAT strikes a good balance among standard accuracy, robustness to seen attacks, and generalization to unseen attacks. Moreover, we demonstrate the flexibility of IJSAT that it can serve as a data augmentation method to improve standard accuracy and assist existing AT methods achieve better performance. 

Future research direction can address the common drawback of on-manifold adversarial training, including DMAT, Adversarial-LA, and the proposed IJSAT. All of them require an additional generator to capture the data manifold and craft adversarial samples. One possible way to drop the generator is to use the feature map of the classifier as the latent space and generate the on-manifold adversarial samples, using the Invertible ResNet \cite{pmlr-v97-behrmann19a}.   
% \blue{add citation to invertible resnet, DMAT, Adversarial-LA}
% Moreover, if we can understand what is the bound of the latent space perturbation in image space, then we can adaptively set the latent space budget for different images to have the same image space budget.   

% In this paper, we exploit the underlying manifold information of images to improve both accuracy and robustness of deep neural networks. We propose Dual Projected Gradient Descent attack to have adversarial perturbations in both image space and latent space, which allows us to generate stronger adversarial samples with both off-manifold and on-manifold perturbations. With this attack, we propose Robust Interpolated Dual Manifold Adversarial Training, which uses both off- and on-manifold interpolated adversarial images in training. With both exact and approximated manifold information, our experiments show that IJSAT boosts robustness and generalizes to novel attacks. % We also evaluate the proposed method in a practical situation where exact manifold information is unavailable and demonstrate its efficacy. 

\subsection*{Acknowledgement}
This work was supported by the DARPA GARD Program HR001119S0026-GARD-FP-052.

{\small
\bibliographystyle{ieee_fullname}
\bibliography{ref}
}

\appendix
\section{Implementation Details}
\subsection{Normalizing Flows} \label{sec:normalizing flows}
To craft on-manifold adversarial samples by our proposed method Joint Space Attack (JSA), we need to first train a flow-based model. We use Glow \cite{DBLP:conf/nips/KingmaD18} in this work. Different from NICE \cite{DBLP:journals/corr/DinhKB14} and RealNVP \cite{DBLP:journals/corr/DinhKB14}, Glow consists of a series of steps of flow, combined in a multi-scale architecture. The multi-scale architecture can reduce the computational complexity of the flow-based model. For each step of the flow, it consists of actnorm, followed by an invertible $1 \times 1$ convolution layer and followed by a coupling layer. LU decomposition is used to calculate the weight matrix of the invertible $ 1 \times 1$ convolution. The number of levels of the Glow we trained is $4$ and we have $8$ steps of flow for each flow. For both CIFAR-10 \cite{krizhevsky2009learning} and OM-ImageNet \cite{lin2020dual} datasets, we train the flow-based model with the original image resolution. 
\subsection{DMAT with CIFAR-10}
DMAT \cite{lin2020dual} requires to use on-manifold images to train the models. They construct an OM-ImageNet dataset and conduct experiments on it. However, if we want to use DMAT in CIFAR-10 dataset, we need to construct an OM-CIFAR10 dataset. Since we would like to have a low-dimensional latent representation, instead of using StyleGAN which has a high dimensional latent representation for Om-ImageNet dataset, we use Spectral Normalization for Generative Adversarial Networks (SNGAN) \cite{miyato2018spectral} as the generator since images only have $32 \times 32 \times 3$ dimensions in CIFAR10 dataset. With the trained SNGAN, we project the training set $\cD_{tr}^o= \{x_i, y_i\}_{i=1}^N$ onto the learned manifold by solving:
\begin{equation}
     z_i = \argmin_{z}~\text{LPIPS}(G(z), x_i) + \norm{G(z) - x_i}_1.
    \label{eq:invert}
\end{equation} 
\noindent where LPIPS is Learned Perceptual Image Patch Similarity ~\cite{zhang2018unreasonable}. The training set contains the projected image with the corresponding latent vectors and labels and the test set contains natural images with corresponding labels. In other words, all the test images are natural (off-manifold) images. All images are normalized into [0, 1], and data augmentations are used during training including random horizontal flipping and 32$\times$32 random cropping with 4-pixel padding. Sample images of the OM-CIFAR10 dataset are shown in Fig. \ref{fig:sample OM-cifar10}. 

In Table \ref{table:dmat comparison cifar10}, we observe a significant improvement with standard accuracy ($3.72\% \uparrow$), and robustness of $L_\infty$ attacks ($7\% \uparrow$) for DMAT+JSA. Since DMAT+JSA has the exact manifold information, the on-manifold adversarial samples preserve the details of the images while DMAT does not. The $L_2$ and JPEG robustness of DMAT is much better than that of DMAT-JSA as there is a large difference between the original image and the projected image.

\begin{table}[t!]
\setlength\tabcolsep{3pt}
\caption{Performance comparison between DMAT trained with approximated manifold (GAN) and exact manifold (Flow) evaluated with the CIFAR-10 dataset.}
% \red{you haven't used bold faces in other tables. we should be consistent, either use it or not. Also we should use the same columns between Tables 3 and 4.}
\label{table:dmat comparison cifar10}
\centering
\small
\scalebox{0.8}{\begin{tabular}{{@{}lr|rrr|rrr|r@{}}}
\toprule
Method          &  Standard  &  $\text{PGD}^{20}$ & AA  & $\text{JSA}^{50}$&  Elastic   &  JPEG &  $L_2$ & Avg\\
\midrule
DMAT \cite{lin2020dual}  &   82.77&	45.01&	34.08&	36.78&	56.72&	38.42&	28.6&	46.05\\
DMAT \cite{lin2020dual} + JSA  &    86.49   & 49.58  & 43.37      &  42.09  & 59.71 &  26.75   &  18.80  & 46.68 \\
Difference  &    $\uparrow$3.72   & $\uparrow$4.57  & $\uparrow$9.29   &  $\uparrow$5.31  & $\uparrow$2.99 &  $\downarrow$11.67   &  $\downarrow$9.8  & $\uparrow$0.63 \\
\bottomrule
\end{tabular}}
\end{table}

\begin{table}[t!]
\setlength\tabcolsep{3pt}
\caption{Classification accuracy against various attacks applied to CIFAR100 dataset. {Bold} values indicate the best performance.} \label{table:cifar100}
\centering
\small
\scalebox{0.85}{\begin{tabular}{@{}lr|rr|rrr|r@{}}
\toprule
Method                &  Standard  &  $\text{PGD}^{20}$ & AA   &  Elastic   &  JPEG &  $L_2$ & Avg\\
\midrule
Normal Training       & 75.73&	0.01&	0&	9.41&	0&	0.43&	14.263\\
AT [PGD-5] \cite{madry2017towards}  & 57.72&	25.62&	22.08&	28.37&	15.37&	10.45&	26.60\\
DMAT \cite{lin2020dual}   & 59.99&	22.49	&19.34&	28.25&	12.56&	8.71&	25.22\\
IAT \cite{lamb2019interpolated}  & \textbf{61.86} &	23.55&	19.69&	31.9&	13.6&	9.93&	26.755 \\
TRADES \cite{zhang2019theoretically}   & 56.46&	28.45&	23.3&	26.46	&16.92&	13.01&	27.43\\
MART \cite{wang2019improving}  & 52.81&	29.49&	\textbf{24.09}&	27.71&	16.08&	14.16&	27.39\\
IJSAT \textbf{(ours)} & 55.83&	\textbf{29.84}&	23.97&	\textbf{33.47} & \textbf{17.28}& \textbf{14.95}& \textbf{29.22}\\
\bottomrule
\end{tabular}}
% \vspace{-3mm}
\end{table}

\begin{figure*}[t!]
  \centering
%   \fbox{\rule[-.5cm]{0cm}{5cm} \rule[-.5cm]{12cm}{0cm}}
  
  \includegraphics[width=0.95\textwidth]{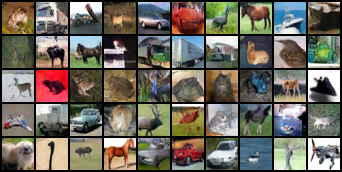}
  \caption{Sample images of OM-CIFAR10 dataset.} \label{fig:sample OM-cifar10}
\end{figure*}

\section{CIFAR-100 Robustness} 
To test our method on a dataset with more classes, we evaluate our methods on CIFAR-100. The results are in Table~\ref{table:cifar100}. We observe similar results as in CIFAR-10: 1) IJSAT has the best overall results; 2) MART, TRADES, and IJSAT achieve comparable $L_\infty$ robustness; 3) IJSAT achieves the best unseen attack robustness. 

\section{Clean Accuracy and PGD-50 Robustness for Multiple Snapshots during Training}
We evaluate the trained models using the PGD-50 for multiple snapshots during training. The results are presented in Figure \ref{fig:plot}. We observe that (i) standard adversarial training leads to degraded standard accuracy (consistent with \cite{tsipras2018robustness}), (ii) on-manifold adversarial samples (methods combined with JSA) improves standard accuracy (consistent with \cite{lin2020dual}), (iii) robustness to PGD attack improves when JSA is added during training and (iv) the improvement is the most significant when DMAT is combined with JSA, which justifies the importance of exact manifold. 
% \begin{filecontents*}{clean_acc.csv}
% Epoch, Normal, AT, IAT, DMAT, Double, RDMAT, RIDMAT
% 2,48.95,45.21,43.81,44.46,23.21,38.79,40.05
% 4,53.34,55.62,55.1,53.32,38.97,48.96,49.46
% 6,71.77,65.89,66.42,69.58,51.31,64.19,57.61
% 8,71.83,67.16,71.23,71.33,57.17,66.91,62.64
% 10,71.24,70.9,72.66,73.7,58.66,68.4,67.82
% 12,74.92,72.9,75.37,76.48,64.78,72.57,70.32
% 14,74.83,73.53,75.91,77.74,67.47,71.25,72.14
% 16,74.79,72.99,76.75,77.38,68.04,73.14,72.66
% 18,74.76,73.12,75.81,77.74,69.04,73.35,72.88
% 20,74.72,73.31,76.81,77.96,68.4,74.72,73.72
% \end{filecontents*}

% \begin{filecontents*}{adv_acc.csv}
% Epoch, Normal, AT, IAT, DMAT, Double, RDMAT, RIDMAT
% 2,0,31.81,25.46,28.61,19.13,28.88,30.51
% 4,0,36.91,28.14,32.21,31.77,34.4,36
% 6,0,44.95,34.81,42.05,38.77,44.86,41.76
% 8,0,42.76,36.57,40.22,42.99,46.75,45.17
% 10,0,45.69,38.8,42.33,43.82,48.54,49.55
% 12,0,45.65,38.09,41.12,46.53,51.59,50.78
% 14,0,42.13,39.26,38.99,48.9,48.69,51.23
% 16,0,40.16,37.61,37.41,48.94,49.3,50.75
% 18,0,39.77,35.64,37.75,49.22,49.4,50.74
% 20,0,38.88,37.15,37.86,48.83,49.42,50.82
% \end{filecontents*}

\begin{figure*}[t!] 
\centering
\begin{tikzpicture}
\begin{groupplot}[
group style={group size=2 by 1},
width=0.5\textwidth,
ymin=0,ymax=100,
xmin=0,xmax=120,
xlabel={Epoch},
ytick = {0, 25, 50, 75, 100}
]

\nextgroupplot [title={Standard Accuracy},
style={line width=0.5pt},
grid=both,
grid style={line width=.1pt, draw=gray!40},
ylabel=Accuracy,
y label style={at={(axis description cs:-0.1,.5)},anchor=south}]

% \addplot[black, mark=o, mark size=2.5pt, mark repeat=10, thick] table[x=Epoch, y=clean, col sep=comma,] {png/supp/cifar10_clean_acc.csv}; \label{normal}
% \addplot[black, mark=diamond*, mark size=2.5pt, mark repeat=10, thick, dashed] table[x=Epoch, y=ia, col sep=comma,] {png/supp/cifar10_clean_acc.csv}; \label{normal_ia}

\addplot[black, thick] table[x=Epoch, y=clean, col sep=comma,] {png/supp/cifar10_clean_acc.csv}; \label{normal}
\addplot[black, dashed, thick] table[x=Epoch, y=ia, col sep=comma,] {png/supp/cifar10_clean_acc.csv}; \label{normal_ia}

\addplot[red, thick] table[x=Epoch, y=at_om, col sep=comma,] {png/supp/cifar10_clean_acc.csv}; \label{AT}
\addplot[red, thick, dashed] table[x=Epoch, y=dmat_double, col sep=comma,] {png/supp/cifar10_clean_acc.csv}; \label{AT_ia}

\addplot[blue, thick] table[x=Epoch, y=dmat_om, col sep=comma,] {png/supp/cifar10_clean_acc.csv}; \label{DMAT}
\addplot[blue, thick, dashed] table[x=Epoch, y=dmat, col sep=comma,] {png/supp/cifar10_clean_acc.csv}; \label{DMAT_ia}

\addplot[green, thick] table[x=Epoch, y=trade, col sep=comma,] {png/supp/cifar10_clean_acc.csv}; \label{TRADE}
\addplot[green, thick, dashed] table[x=Epoch, y=trade_ia, col sep=comma,] {png/supp/cifar10_clean_acc.csv}; \label{TRADE_ia}

\addplot[brown, thick] table[x=Epoch, y=mart, col sep=comma,] {png/supp/cifar10_clean_acc.csv}; \label{MART}
\addplot[brown, thick, dashed] table[x=Epoch, y=mart_ia, col sep=comma,] {png/supp/cifar10_clean_acc.csv}; \label{MART_ia}

% \addplot[red, mark=o, mark size=2.5pt, mark repeat=10, thick] table[x=Epoch, y=at_om, col sep=comma,] {png/supp/cifar10_clean_acc.csv}; \label{AT}
% \addplot[red, mark=diamond*, mark size=2.5pt, mark repeat=10, thick, dashed] table[x=Epoch, y=dmat_double, col sep=comma,] {png/supp/cifar10_clean_acc.csv}; \label{AT_ia}

% \addplot[red, mark=+, mark size=1.5pt] table[x=Epoch, y=AT, col sep=comma,] {clean_acc.csv};\label{image_sgd}
% \addplot[orange, mark=+, mark size=1.5pt] table[x=Epoch, y=IAT, col sep=comma,] {clean_acc.csv};\label{iat}
% \addplot[blue, mark=+, mark size=1.5pt] table[x=Epoch, y=DMAT, col sep=comma,] {clean_acc.csv};\label{dmat}

% \addplot[green, mark=diamond*, mark size=1.5pt] table[x=Epoch, y=Double, col sep=comma,] {clean_acc.csv};\label{double}
% \addplot[purple, mark=diamond*, mark size=1.5pt] table[x=Epoch, y=RDMAT, col sep=comma,] {clean_acc.csv};\label{double_dmat}
% \addplot[brown, mark=diamond*, mark size=1.5pt] table[x=Epoch, y=RIDMAT, col sep=comma,] {clean_acc.csv};\label{ridmat}
\coordinate (top) at (rel axis cs:0,1);

\nextgroupplot[title={Accuracy on PGD attack},
style={line width=0.5pt},
grid=both,
grid style={line width=.1pt, draw=gray!40}
]

\addplot[black, thick] table[x=Epoch, y=clean, col sep=comma,] {png/supp/cifar10_adv_acc.csv}; 
\addplot[black, dashed, thick] table[x=Epoch, y=ia, col sep=comma,] {png/supp/cifar10_adv_acc.csv}; 

\addplot[red, thick] table[x=Epoch, y=at_om, col sep=comma,] {png/supp/cifar10_adv_acc.csv}; 
\addplot[red, thick, dashed] table[x=Epoch, y=dmat_double, col sep=comma,] {png/supp/cifar10_adv_acc.csv}; 

\addplot[blue, thick] table[x=Epoch, y=dmat_om, col sep=comma,] {png/supp/cifar10_adv_acc.csv}; 
\addplot[blue, thick, dashed] table[x=Epoch, y=dmat, col sep=comma,] {png/supp/cifar10_adv_acc.csv}; 

\addplot[green, thick] table[x=Epoch, y=trade, col sep=comma,] {png/supp/cifar10_adv_acc.csv}; 
\addplot[green, thick, dashed] table[x=Epoch, y=trade_ia, col sep=comma,] {png/supp/cifar10_adv_acc.csv};

\addplot[brown, thick] table[x=Epoch, y=mart, col sep=comma,] {png/supp/cifar10_adv_acc.csv}; 
\addplot[brown, thick, dashed] table[x=Epoch, y=mart_ia, col sep=comma,] {png/supp/cifar10_adv_acc.csv}; 

% \addplot[black, mark=+, mark size=1.5pt] table[x=Epoch, y=Normal, col sep=comma,] {adv_acc.csv};
% \addplot[red, mark=+, mark size=1.5pt] table[x=Epoch, y=AT, col sep=comma,] {adv_acc.csv}; 
% \addplot[orange, mark=+, mark size=1.5pt] table[x=Epoch, y=IAT, col sep=comma,] {adv_acc.csv};
% \addplot[blue, mark=+, mark size=1.5pt] table[x=Epoch, y=DMAT, col sep=comma,] {adv_acc.csv}; 
% \addplot[green, mark=diamond*, mark size=1.5pt] table[x=Epoch, y=Double, col sep=comma,] {adv_acc.csv}; 
% \addplot[purple, mark=diamond*, mark size=1.5pt] table[x=Epoch, y=RDMAT, col sep=comma,] {adv_acc.csv};
% \addplot[brown, mark=diamond*, mark size=1.5pt] table[x=Epoch, y=RIDMAT, col sep=comma,] {adv_acc.csv};

\coordinate (bot) at (rel axis cs:1,0);
\end{groupplot}
legend
  \path (top|-current bounding box.south)--
        coordinate(legendpos)
        (bot|-current bounding box.south);
  \matrix[
      matrix of nodes,
      anchor=north,
      draw,
      inner sep=0.2em,
    %   column 1/.style={nodes={align=center}},
    %   column 2/.style={nodes={anchor=base west, font=\tiny}},
    %   column 3/.style={nodes={align=center}},
    %   column 4/.style={nodes={anchor=base west, font=\tiny}},
    ]at([yshift=-1ex]legendpos)
    { \ref{normal}& Normal&[5pt]
      \ref{normal_ia}& Normal + IA &[5pt] 
      \ref{AT}& AT &[5pt] 
      \ref{AT_ia}& AT + IA &[5pt] 
      \ref{DMAT}& DMAT &[5pt] \\
      \ref{DMAT_ia}& DMAT + IA &[5pt]
      \ref{TRADE}& TRADE &[5pt] 
      \ref{TRADE_ia}& TRADE + IA &[5pt]
      \ref{MART}& MART &[5pt]
      \ref{MART_ia}& MART + IA \\};
\end{tikzpicture}
\caption{Clean accuracy and PGD-50 robustness during training with CIFAR-10 dataset. Using IA during improves both standard accuracy and robustness.  Left: standard accuracy. Right: classification accuracy when the trained models are attacked by PGD attack.}
\label{fig:plot}
\end{figure*}
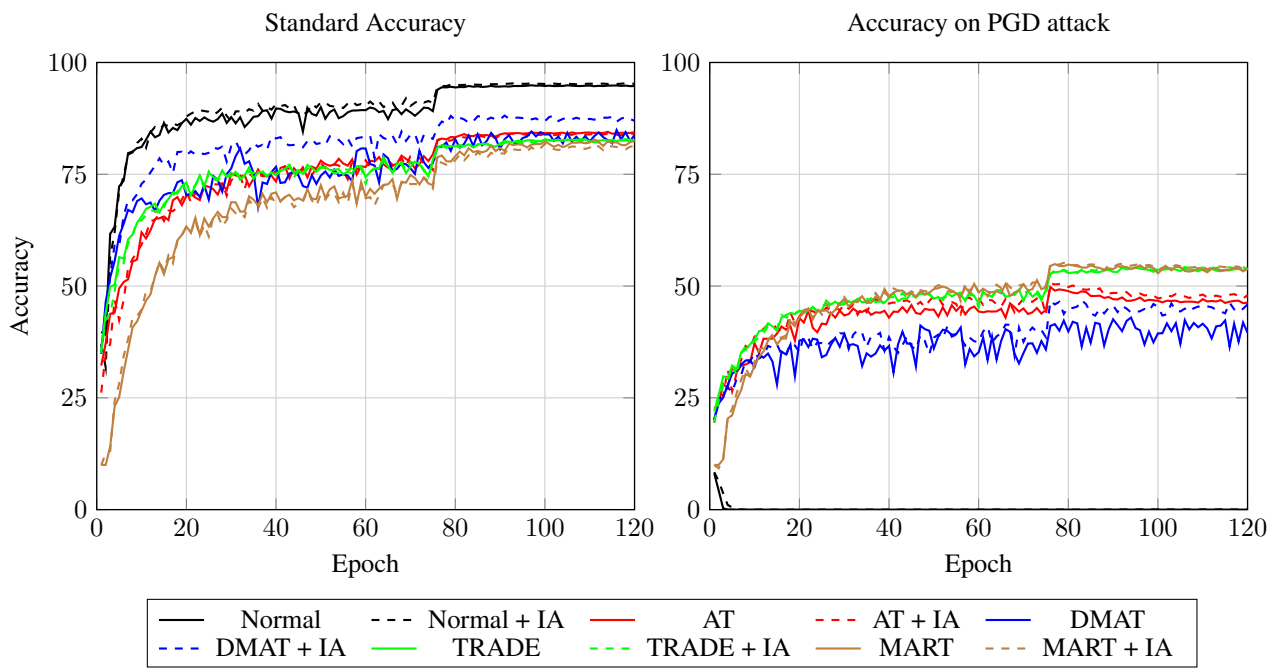

\section{JSA at different level of flow}
As we mentioned in Sec. \ref{sec:normalizing flows}, the Glow model we trained has multi-scale architecture. Mathematically, the latent vector is the following:

\begin{equation} \label{eq: composite}
    g^{-1}(x) = g_1^{-1} \circ g_2^{-1} \circ g_3^{-1} \circ g_4^{-1} (x) = (z_1, z_2, z_3, z_4) = z
\end{equation}
\noindent where $g$ is the flow-based model which consists of 4 levels. We perturb the whole latent vector $z$ when we craft an adversarial sample from JSA. Instead of attacking the whole latent vector, we would like to visualize the adversarial samples when only one level of the latent vector is attacked. Mathematically, 
\begin{equation}
    \max_{\lambda_i \in \Lambda} \cL (f_{\theta}(g(z_1, \dots, z_i + \lambda_i, z_4)), y_{\text{true}}),
    \label{eq:on-manifold_supp}
\end{equation}
\noindent for some flow level $i$.  The visualization results are shown in Fig. \ref{fig:layer}. We can observe that the perturbation pattern changes from coarse to fine in semantic details. In other words, the perturbation pattern is similar to PGD or Gabor attack when it is at level 1 and the perturbation becomes more similar to the subject (leopard in this example) when it is attacked at a higher level. Therefore, when the whole latent vector is attacked, all these perturbations will be accumulated and make This somehow explains why on-manifold adversarial samples from JSA could improve the robustness of various attacks. 

\begin{figure*}[t!]
    \centering
    \captionsetup[subfigure]{justification=centering}
 
    % \begin{subfigure}[t]{0.24\textwidth}
    %     \raisebox{-\height}{\includegraphics[width=\textwidth]{png/supp/layer/batch_5074_inputs.png}}
    %     \caption*{Original}
    % \end{subfigure}
    \begin{subfigure}[t]{0.16\textwidth}
        \raisebox{-\height}{\includegraphics[width=0.99\textwidth]{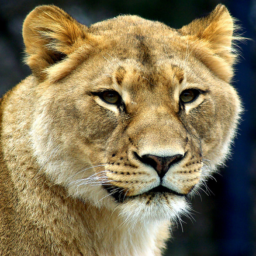}}\\
        \raisebox{-\height}{\includegraphics[width=0.99\textwidth]{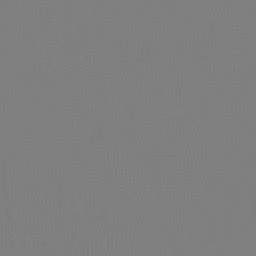}}
        \caption*{IA ($\eta = 0.02$) at level 1}
    \end{subfigure}
    \begin{subfigure}[t]{0.16\textwidth}
        \raisebox{-\height}{\includegraphics[width=0.99\textwidth]{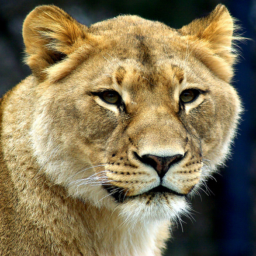}}\\
        \raisebox{-\height}{\includegraphics[width=0.99\textwidth]{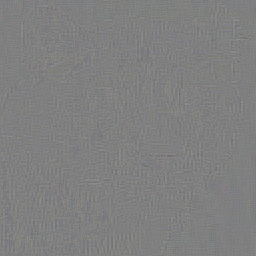}}
        \caption*{IA ($\eta = 0.04$) at level 1}
    \end{subfigure}
    \begin{subfigure}[t]{0.16\textwidth}
        \raisebox{-\height}{\includegraphics[width=0.99\textwidth]{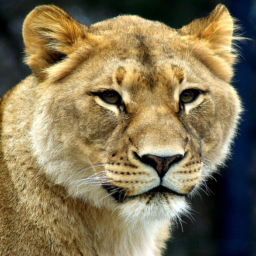}}\\
        \raisebox{-\height}{\includegraphics[width=0.99\textwidth]{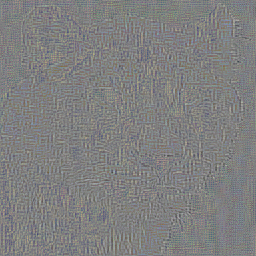}}
        \caption*{IA ($\eta = 0.1$) at level 1}
    \end{subfigure}
    \begin{subfigure}[t]{0.16\textwidth}
        \raisebox{-\height}{\includegraphics[width=0.99\textwidth]{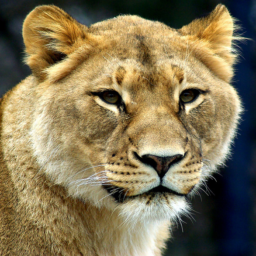}}\\
        \raisebox{-\height}{\includegraphics[width=0.99\textwidth]{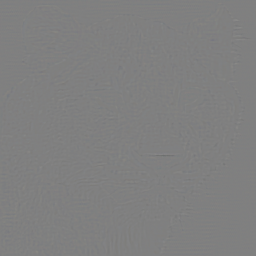}}
        \caption*{IA ($\eta = 0.02$) at level 2}
    \end{subfigure}
    \begin{subfigure}[t]{0.16\textwidth}
        \raisebox{-\height}{\includegraphics[width=0.99\textwidth]{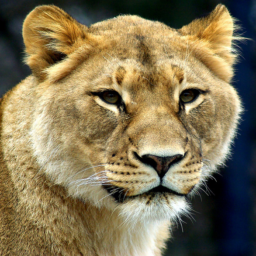}}\\
        \raisebox{-\height}{\includegraphics[width=0.99\textwidth]{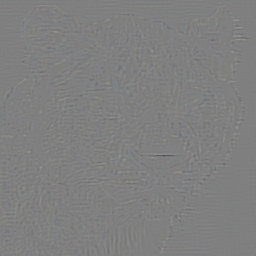}}
        \caption*{IA ($\eta = 0.04$) at level 2}
    \end{subfigure}
    \begin{subfigure}[t]{0.16\textwidth}
        \raisebox{-\height}{\includegraphics[width=0.99\textwidth]{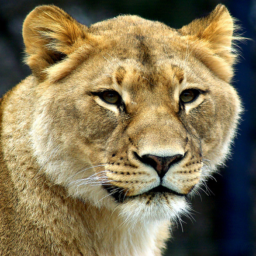}}\\
        \raisebox{-\height}{\includegraphics[width=0.99\textwidth]{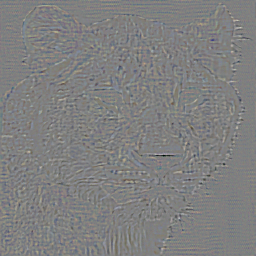}}
        \caption*{IA ($\eta = 0.1$) at level 2}
    \end{subfigure}
    \\
    % \begin{subfigure}[t]{0.24\textwidth}
    %     \raisebox{-\height}{\includegraphics[width=\textwidth]{png/supp/layer/batch_5074_inputs.png}}
    %     \caption*{Original}
    % \end{subfigure}
    \begin{subfigure}[t]{0.16\textwidth}
        \raisebox{-\height}{\includegraphics[width=0.99\textwidth]{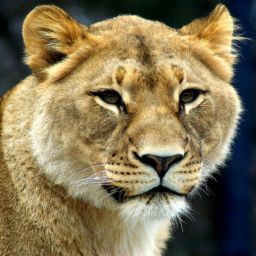}}\\
        \raisebox{-\height}{\includegraphics[width=0.99\textwidth]{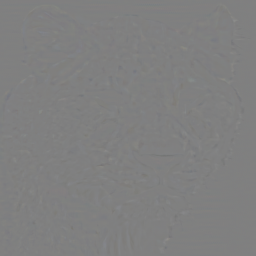}}
        \caption*{IA ($\eta = 0.02$) at level 3}
    \end{subfigure}
    \begin{subfigure}[t]{0.16\textwidth}
        \raisebox{-\height}{\includegraphics[width=0.99\textwidth]{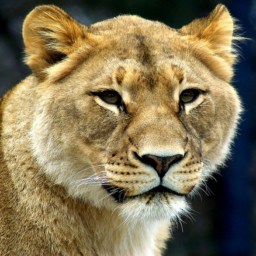}}\\
        \raisebox{-\height}{\includegraphics[width=0.99\textwidth]{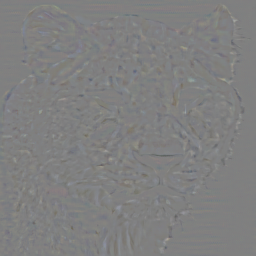}}
        \caption*{IA ($\eta = 0.04$) at level 3}
    \end{subfigure}
    \begin{subfigure}[t]{0.16\textwidth}
        \raisebox{-\height}{\includegraphics[width=0.99\textwidth]{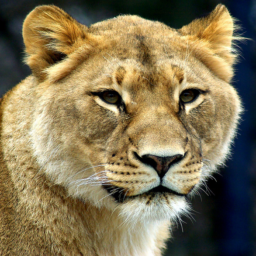}}\\
        \raisebox{-\height}{\includegraphics[width=0.99\textwidth]{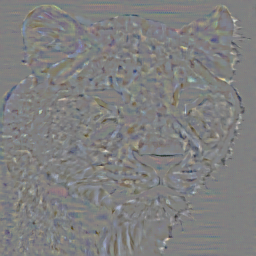}}
        \caption*{IA ($\eta = 0.1$) at level 3}
    \end{subfigure}
    \begin{subfigure}[t]{0.16\textwidth}
        \raisebox{-\height}{\includegraphics[width=0.99\textwidth]{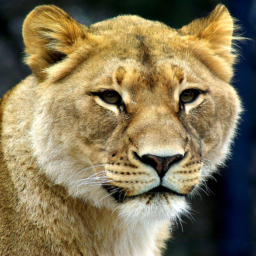}}\\
        \raisebox{-\height}{\includegraphics[width=0.99\textwidth]{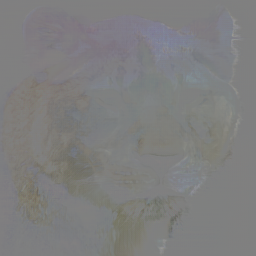}}
        \caption*{IA ($\eta = 0.02$) at level 4}
    \end{subfigure}
    \begin{subfigure}[t]{0.16\textwidth}
        \raisebox{-\height}{\includegraphics[width=0.99\textwidth]{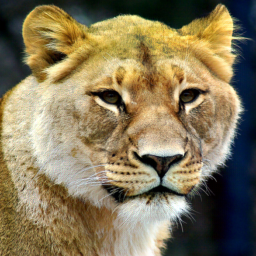}}\\
        \raisebox{-\height}{\includegraphics[width=0.99\textwidth]{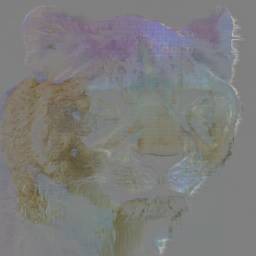}}
        \caption*{IA ($\eta = 0.04$) at level 4}
    \end{subfigure}
    \begin{subfigure}[t]{0.16\textwidth}
        \raisebox{-\height}{\includegraphics[width=0.99\textwidth]{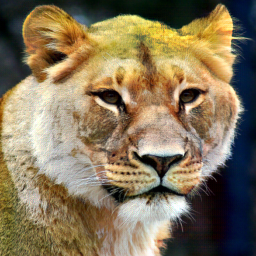}}\\
        \raisebox{-\height}{\includegraphics[width=0.99\textwidth]{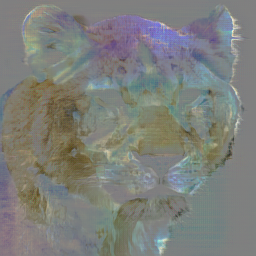}}
        \caption*{IA ($\eta = 0.1$) at level 4}
    \end{subfigure}

    \caption{Visualization of adversarial samples from IA with various perturbation budgets at different flow levels. The differences are magnified for visualization purpose.} \label{fig:layer}
    % \vspace{-1mm}
\end{figure*}

\section{Class Activation Map of JSTM}
To understand how JSA samples improves standard accuracy and robustness, we can use a class activation map (CAM) to obtain a good visual explanation. In this paper, we use GRAD-CAM \cite{selvaraju2017grad} to visualize. The CAM results are shown in Fig. \ref{fig:cam_supp}. For the model with normal training, we can see the CAM is sparse on the semantic region even for the original image. With JSA, the CAM is more concentrated on the semantic region. For PGD, JSA, and Elastic attacks to model with normal training, these attacks break the classifier and make the CAM wrong. For normal training with JSA, it successfully defends JSA and Elastic attack and has a meaningful CAM while it is broken by PGD attack. For the existing adversarial training method, using JSA could have CAM more concentrated on the semantic regions (See the row AT and TRADE). This shows that using IA during training can help the classifier to concentrate more on the semantic regions than not using IA even both models classify correctly. 

\begin{figure*}[t!]
    \centering
    \captionsetup[subfigure]{justification=centering}
    \begin{subfigure}[t]{0.22\textwidth}
        \raisebox{-6ex}{\rotatebox[origin=c]{90}{\text{Normal}}}
    \end{subfigure}
    \hspace{-3.3cm}
    % \begin{subfigure}[t]{0.2\textwidth}
    %     \raisebox{-\height}{\includegraphics[width=\textwidth]{png/supp/layer/batch_5074_inputs.png}}
    %     \caption*{Original}
    % \end{subfigure}
    \begin{subfigure}[t]{0.24\textwidth}
        \raisebox{-\height}{\includegraphics[width=0.48\textwidth]{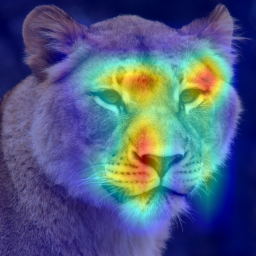}} \hspace{-2mm}
        \raisebox{-\height}{\includegraphics[width=0.48\textwidth]{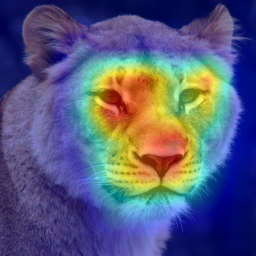}} 
        % \caption*{Original}
    \end{subfigure}
    \hspace{-2mm}
    \begin{subfigure}[t]{0.24\textwidth}
        \raisebox{-\height}{\includegraphics[width=0.48\textwidth]{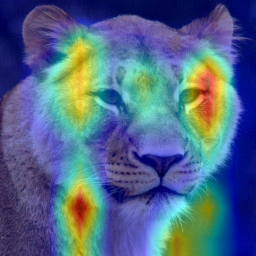}} \hspace{-2mm}
        \raisebox{-\height}{\includegraphics[width=0.48\textwidth]{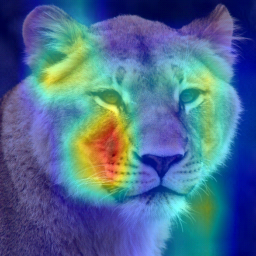}} 
        % \caption*{PGD}
    \end{subfigure}
    \hspace{-2mm}
    \begin{subfigure}[t]{0.24\textwidth}
        \raisebox{-\height}{\includegraphics[width=0.48\textwidth]{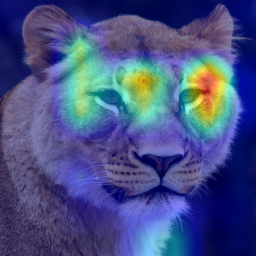}} \hspace{-2mm}
        \raisebox{-\height}{\includegraphics[width=0.48\textwidth]{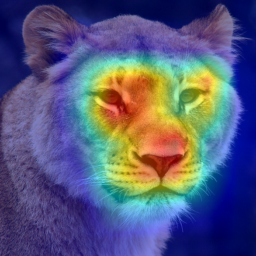}} 
        % \caption*{IA}
    \end{subfigure}
    \hspace{-2mm}
    \begin{subfigure}[t]{0.24\textwidth}
        \raisebox{-\height}{\includegraphics[width=0.48\textwidth]{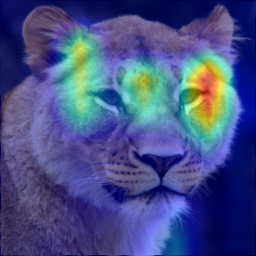}} \hspace{-2mm}
        \raisebox{-\height}{\includegraphics[width=0.48\textwidth]{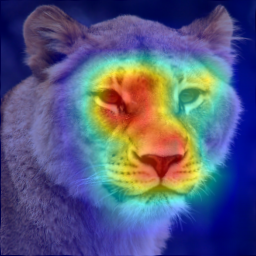}} 
        % \caption*{Elastic}
    \end{subfigure}
    \\
    \begin{subfigure}[t]{0.22\textwidth}
        \raisebox{-6ex}{\rotatebox[origin=c]{90}{\text{AT}}}
    \end{subfigure}
    \hspace{-3.3cm}
    % \begin{subfigure}[t]{0.2\textwidth}
    %     \raisebox{-\height}{\includegraphics[width=\textwidth]{png/supp/layer/batch_5074_inputs.png}}
    %     \caption*{Original}
    % \end{subfigure}
    \begin{subfigure}[t]{0.24\textwidth}
        \raisebox{-\height}{\includegraphics[width=0.48\textwidth]{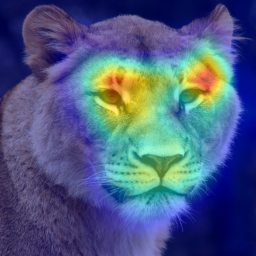}} \hspace{-2mm}
        \raisebox{-\height}{\includegraphics[width=0.48\textwidth]{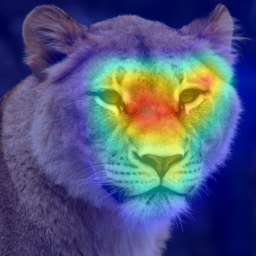}} 
        % \caption*{Original}
    \end{subfigure}
    \hspace{-2mm}
    \begin{subfigure}[t]{0.24\textwidth}
        \raisebox{-\height}{\includegraphics[width=0.48\textwidth]{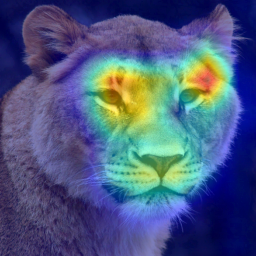}} \hspace{-2mm}
        \raisebox{-\height}{\includegraphics[width=0.48\textwidth]{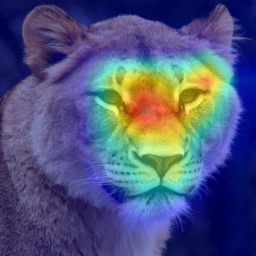}} 
        % \caption*{PGD}
    \end{subfigure}
    \hspace{-2mm}
    \begin{subfigure}[t]{0.24\textwidth}
        \raisebox{-\height}{\includegraphics[width=0.48\textwidth]{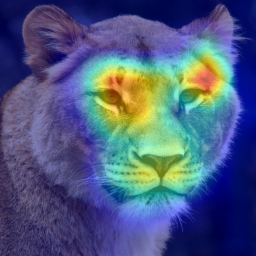}} \hspace{-2mm}
        \raisebox{-\height}{\includegraphics[width=0.48\textwidth]{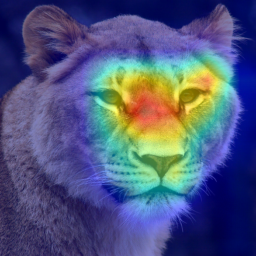}} 
        % \caption*{IA}
    \end{subfigure}
    \hspace{-2mm}
    \begin{subfigure}[t]{0.24\textwidth}
        \raisebox{-\height}{\includegraphics[width=0.48\textwidth]{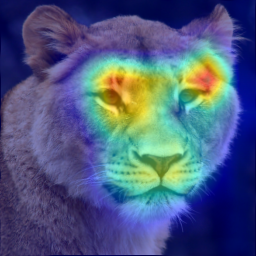}} \hspace{-2mm}
        \raisebox{-\height}{\includegraphics[width=0.48\textwidth]{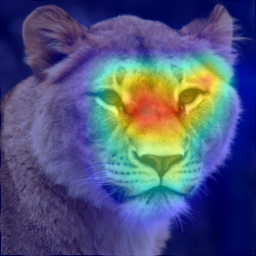}} 
        % \caption*{Elastic}
    \end{subfigure}
    \\
    \begin{subfigure}[t]{0.22\textwidth}
        \raisebox{-6ex}{\rotatebox[origin=c]{90}{\text{DMAT}}}
    \end{subfigure}
    \hspace{-3.3cm}
    % \begin{subfigure}[t]{0.2\textwidth}
    %     \raisebox{-\height}{\includegraphics[width=\textwidth]{png/supp/layer/batch_5074_inputs.png}}
    %     \caption*{Original}
    % \end{subfigure}
    \begin{subfigure}[t]{0.24\textwidth}
        \raisebox{-\height}{\includegraphics[width=0.48\textwidth]{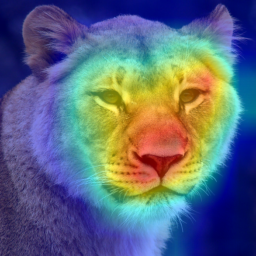}} \hspace{-2mm}
        \raisebox{-\height}{\includegraphics[width=0.48\textwidth]{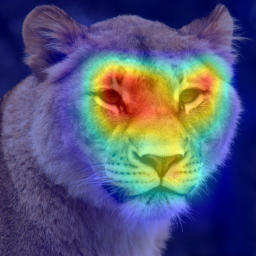}} 
        % \caption*{Original}
    \end{subfigure}
    \hspace{-2mm}
    \begin{subfigure}[t]{0.24\textwidth}
        \raisebox{-\height}{\includegraphics[width=0.48\textwidth]{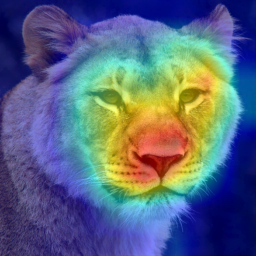}} \hspace{-2mm}
        \raisebox{-\height}{\includegraphics[width=0.48\textwidth]{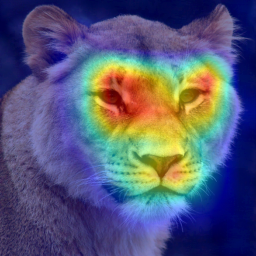}} 
        % \caption*{PGD}
    \end{subfigure}
    \hspace{-2mm}
    \begin{subfigure}[t]{0.24\textwidth}
        \raisebox{-\height}{\includegraphics[width=0.48\textwidth]{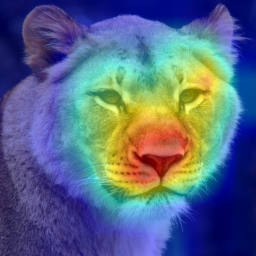}} \hspace{-2mm}
        \raisebox{-\height}{\includegraphics[width=0.48\textwidth]{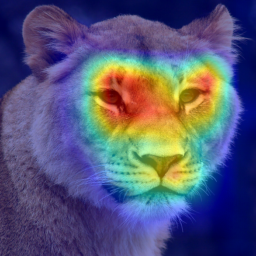}} 
        % \caption*{IA}
    \end{subfigure}
    \hspace{-2mm}
    \begin{subfigure}[t]{0.24\textwidth}
        \raisebox{-\height}{\includegraphics[width=0.48\textwidth]{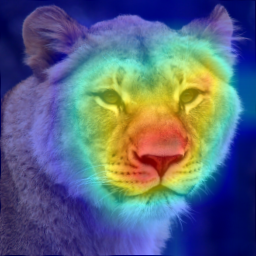}} \hspace{-2mm}
        \raisebox{-\height}{\includegraphics[width=0.48\textwidth]{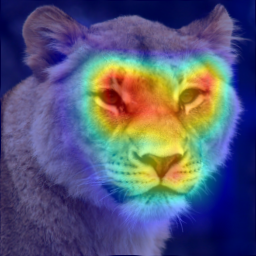}} 
        % \caption*{Elastic}
    \end{subfigure}
    \\
    \begin{subfigure}[t]{0.22\textwidth}
        \raisebox{-6ex}{\rotatebox[origin=c]{90}{\text{TRADE}}}
    \end{subfigure}
    \hspace{-3.3cm}
    % \begin{subfigure}[t]{0.2\textwidth}
    %     \raisebox{-\height}{\includegraphics[width=\textwidth]{png/supp/layer/batch_5074_inputs.png}}
    %     \caption*{Original}
    % \end{subfigure}
    \begin{subfigure}[t]{0.24\textwidth}
        \raisebox{-\height}{\includegraphics[width=0.48\textwidth]{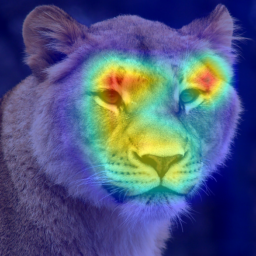}} \hspace{-2mm}
        \raisebox{-\height}{\includegraphics[width=0.48\textwidth]{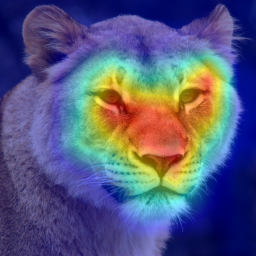}} 
        % \caption*{Original}
    \end{subfigure}
    \hspace{-2mm}
    \begin{subfigure}[t]{0.24\textwidth}
        \raisebox{-\height}{\includegraphics[width=0.48\textwidth]{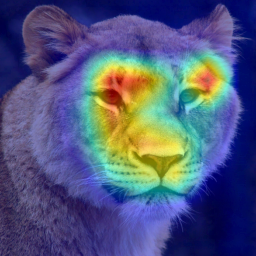}} \hspace{-2mm}
        \raisebox{-\height}{\includegraphics[width=0.48\textwidth]{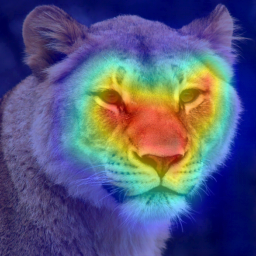}} 
        % \caption*{PGD}
    \end{subfigure}
    \hspace{-2mm}
    \begin{subfigure}[t]{0.24\textwidth}
        \raisebox{-\height}{\includegraphics[width=0.48\textwidth]{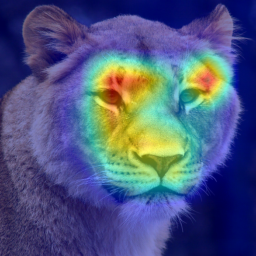}} \hspace{-2mm}
        \raisebox{-\height}{\includegraphics[width=0.48\textwidth]{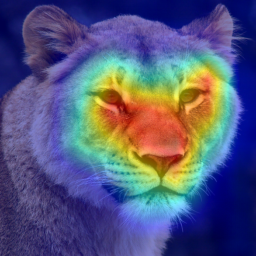}} 
        % \caption*{IA}
    \end{subfigure}
    \hspace{-2mm}
    \begin{subfigure}[t]{0.24\textwidth}
        \raisebox{-\height}{\includegraphics[width=0.48\textwidth]{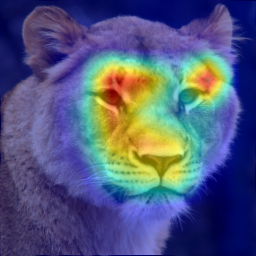}} \hspace{-2mm}
        \raisebox{-\height}{\includegraphics[width=0.48\textwidth]{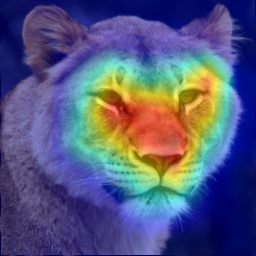}} 
        % \caption*{Elastic}
    \end{subfigure}
    \\
    \begin{subfigure}[t]{0.22\textwidth}
        \raisebox{-6ex}{\rotatebox[origin=c]{90}{\text{MART}}}
    \end{subfigure}
    \hspace{-3.3cm}
    % \begin{subfigure}[t]{0.2\textwidth}
    %     \raisebox{-\height}{\includegraphics[width=\textwidth]{png/supp/layer/batch_5074_inputs.png}}
    %     \caption*{Original}
    % \end{subfigure}
    \begin{subfigure}[t]{0.24\textwidth}
        \raisebox{-\height}{\includegraphics[width=0.48\textwidth]{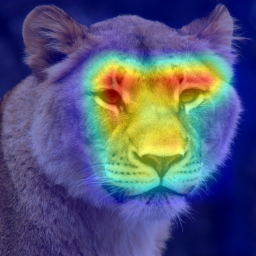}} \hspace{-2mm}
        \raisebox{-\height}{\includegraphics[width=0.48\textwidth]{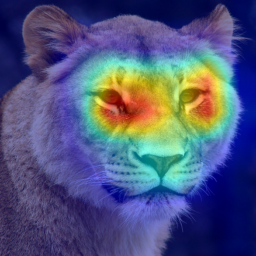}} 
        \caption*{Original}
    \end{subfigure}
    \hspace{-2mm}
    \begin{subfigure}[t]{0.24\textwidth}
        \raisebox{-\height}{\includegraphics[width=0.48\textwidth]{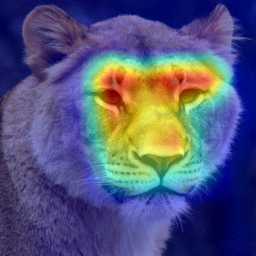}} \hspace{-2mm}
        \raisebox{-\height}{\includegraphics[width=0.48\textwidth]{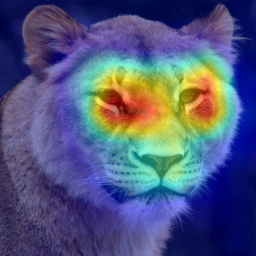}} 
        \caption*{PGD}
    \end{subfigure}
    \hspace{-2mm}
    \begin{subfigure}[t]{0.24\textwidth}
        \raisebox{-\height}{\includegraphics[width=0.48\textwidth]{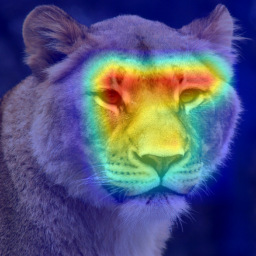}} \hspace{-2mm}
        \raisebox{-\height}{\includegraphics[width=0.48\textwidth]{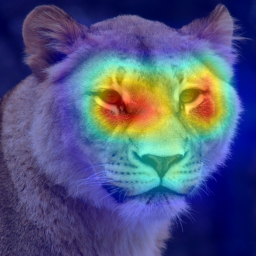}} 
        \caption*{IA}
    \end{subfigure}
    \hspace{-2mm}
    \begin{subfigure}[t]{0.24\textwidth}
        \raisebox{-\height}{\includegraphics[width=0.48\textwidth]{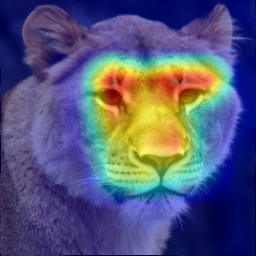}} \hspace{-2mm}
        \raisebox{-\height}{\includegraphics[width=0.48\textwidth]{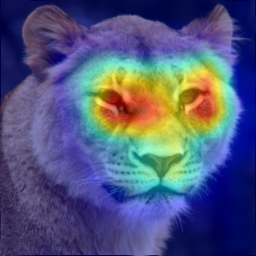}} 
        \caption*{Elastic}
    \end{subfigure}

    \caption{Visualization of CAM for different methods with and without IA augmentation. Each row represents a defense model and each column represents an attack. For each combination of attack and defense method, there are two images. The left one is the original method and the right one is that combined with IA. The heap map has a higher confidence score and overlap more with the semantic meaningful region (the face of teh leopard) when IA augmentation is used.} \label{fig:cam_supp}
    % \vspace{-1mm}
\end{figure*}

\section{Additional Figures for JSA Images}
\begin{figure*}[t]
% \vspace{-2mm}
\setlength\tabcolsep{0pt}
\renewcommand{\arraystretch}{0}
\scalebox{1}{\begin{tabular}{cccccccc}
\begin{subfigure}[t]{0.123\textwidth}
\includegraphics[width=\textwidth]{png/jsa/0_iter_0_inputs_0.png}
\end{subfigure} &
% \hspace{-2mm}
\begin{subfigure}[t]{0.123\textwidth}
\includegraphics[width=\textwidth]{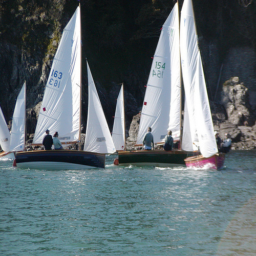}
\end{subfigure} &
% \hspace{-2mm}
% \\ \vspace{-1mm}
\begin{subfigure}[t]{0.123\textwidth}
\includegraphics[width=\textwidth]{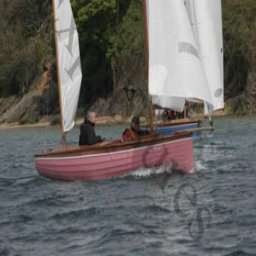}
\end{subfigure} &
% \hspace{-2mm}
% \hspace{-1.8mm}
\begin{subfigure}[t]{0.123\textwidth}
\includegraphics[width=\textwidth]{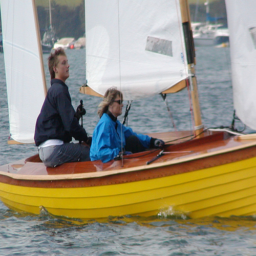}
\end{subfigure} &
% \hspace{-2mm}
\begin{subfigure}[t]{0.123\textwidth}
\includegraphics[width=\textwidth]{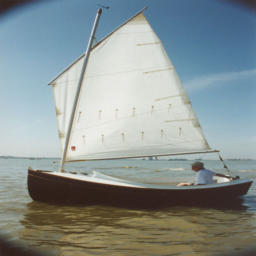}
\end{subfigure} &
\begin{subfigure}[t]{0.123\textwidth}
\includegraphics[width=\textwidth]{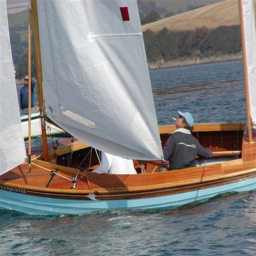}
\end{subfigure} &
% \hspace{-2mm}
% \hspace{-1.8mm}
\begin{subfigure}[t]{0.123\textwidth}
\includegraphics[width=\textwidth]{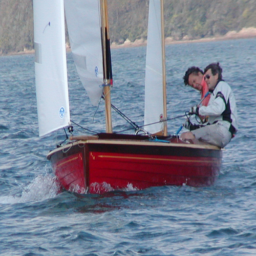}
\end{subfigure} &
\hspace{-2mm}
\begin{subfigure}[t]{0.123\textwidth}
\includegraphics[width=\textwidth]{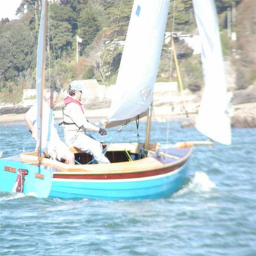}
\end{subfigure}
\\
\begin{subfigure}[t]{0.123\textwidth}
\includegraphics[width=\textwidth]{png/jsa/0_iter_0_inputs_0_jsa.png}
\end{subfigure} &
% \hspace{-2mm}
\begin{subfigure}[t]{0.123\textwidth}
\includegraphics[width=\textwidth]{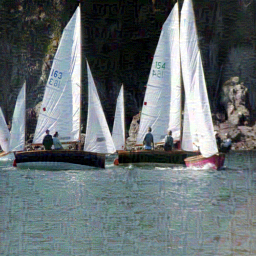}
\end{subfigure} &
% \hspace{-2mm}
% \\ \vspace{-1mm}
\begin{subfigure}[t]{0.123\textwidth}
\includegraphics[width=\textwidth]{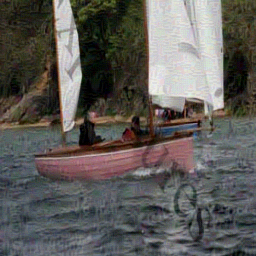}
\end{subfigure} &
% \hspace{-2mm}
% \hspace{-1.8mm}
\begin{subfigure}[t]{0.123\textwidth}
\includegraphics[width=\textwidth]{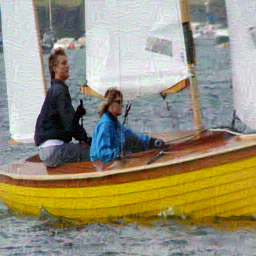}
\end{subfigure} &
% \hspace{-2mm}
\begin{subfigure}[t]{0.123\textwidth}
\includegraphics[width=\textwidth]{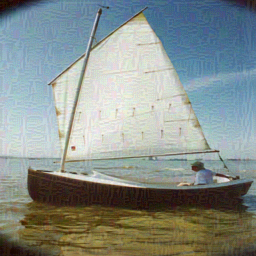}
\end{subfigure} &
\begin{subfigure}[t]{0.123\textwidth}
\includegraphics[width=\textwidth]{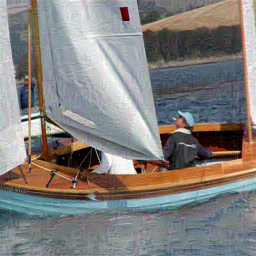}
\end{subfigure} &
% \hspace{-2mm}
% \hspace{-1.8mm}
\begin{subfigure}[t]{0.123\textwidth}
\includegraphics[width=\textwidth]{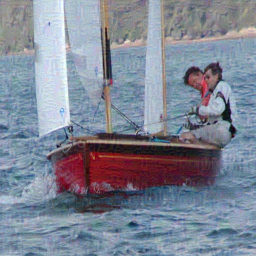}
\end{subfigure} &
\hspace{-2mm}
\begin{subfigure}[t]{0.123\textwidth}
\includegraphics[width=\textwidth]{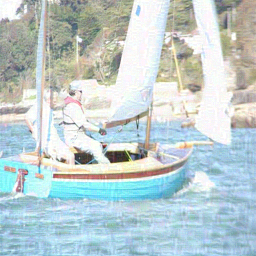}
\end{subfigure}
\\
\begin{subfigure}[t]{0.123\textwidth}
\includegraphics[width=\textwidth]{png/jsa/0_iter_0_inputs_0_jsa_diff.png}
\end{subfigure} &
% \hspace{-2mm}
\begin{subfigure}[t]{0.123\textwidth}
\includegraphics[width=\textwidth]{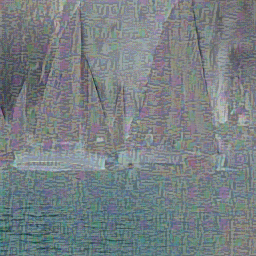}
\end{subfigure} &
% \hspace{-2mm}
% \\ \vspace{-1mm}
\begin{subfigure}[t]{0.123\textwidth}
\includegraphics[width=\textwidth]{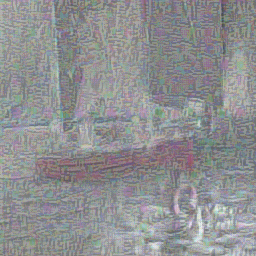}
\end{subfigure} &
% \hspace{-2mm}
% \hspace{-1.8mm}
\begin{subfigure}[t]{0.123\textwidth}
\includegraphics[width=\textwidth]{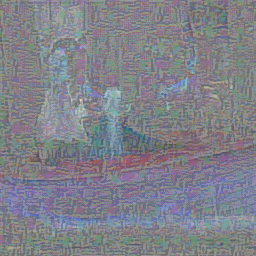}
\end{subfigure} &
% \hspace{-2mm}
\begin{subfigure}[t]{0.123\textwidth}
\includegraphics[width=\textwidth]{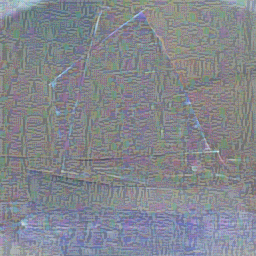}
\end{subfigure} &
\begin{subfigure}[t]{0.123\textwidth}
\includegraphics[width=\textwidth]{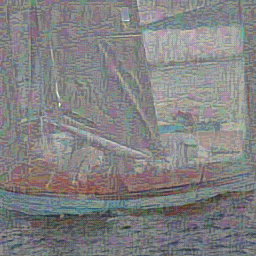}
\end{subfigure} &
% \hspace{-2mm}
% \hspace{-1.8mm}
\begin{subfigure}[t]{0.123\textwidth}
\includegraphics[width=\textwidth]{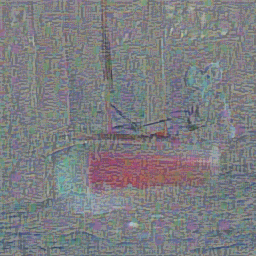}
\end{subfigure} &
% \hspace{-2mm}
\begin{subfigure}[t]{0.123\textwidth}
\includegraphics[width=\textwidth]{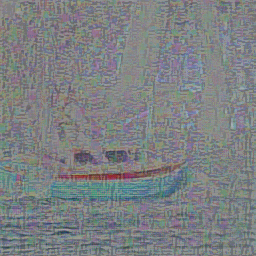}
\end{subfigure}
\vspace{2mm}
\\
\begin{subfigure}[t]{0.123\textwidth}
\includegraphics[width=\textwidth]{png/jsa/0_iter_100_inputs_0.png}
\end{subfigure} &
% \hspace{-2mm}
\begin{subfigure}[t]{0.123\textwidth}
\includegraphics[width=\textwidth]{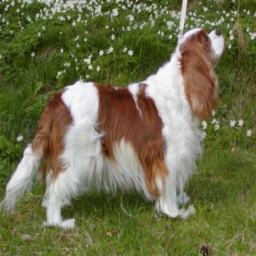}
\end{subfigure} &
% \hspace{-2mm}
% \\ \vspace{-1mm}
\begin{subfigure}[t]{0.123\textwidth}
\includegraphics[width=\textwidth]{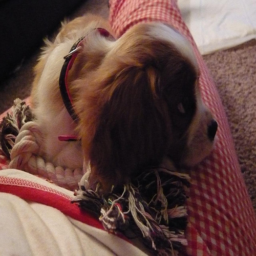}
\end{subfigure} &
% \hspace{-2mm}
% \hspace{-1.8mm}
\begin{subfigure}[t]{0.123\textwidth}
\includegraphics[width=\textwidth]{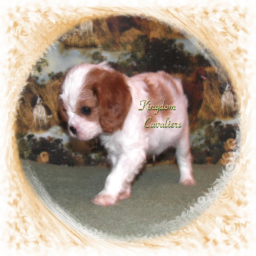}
\end{subfigure} &
% \hspace{-2mm}
\begin{subfigure}[t]{0.123\textwidth}
\includegraphics[width=\textwidth]{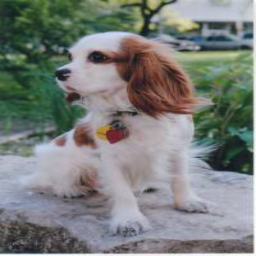}
\end{subfigure} &
\begin{subfigure}[t]{0.123\textwidth}
\includegraphics[width=\textwidth]{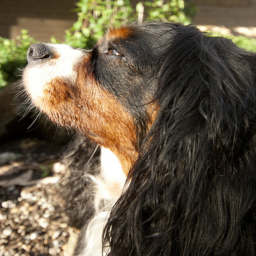}
\end{subfigure} &
% \hspace{-2mm}
% \hspace{-1.8mm}
\begin{subfigure}[t]{0.123\textwidth}
\includegraphics[width=\textwidth]{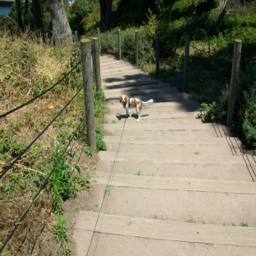}
\end{subfigure} &
\hspace{-2mm}
\begin{subfigure}[t]{0.123\textwidth}
\includegraphics[width=\textwidth]{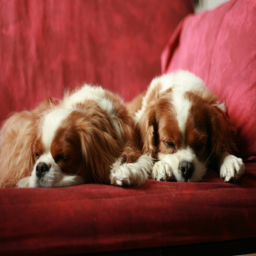}
\end{subfigure}
\\
\begin{subfigure}[t]{0.123\textwidth}
\includegraphics[width=\textwidth]{png/jsa/0_iter_100_inputs_0_jsa.png}
\end{subfigure} &
% \hspace{-2mm}
\begin{subfigure}[t]{0.123\textwidth}
\includegraphics[width=\textwidth]{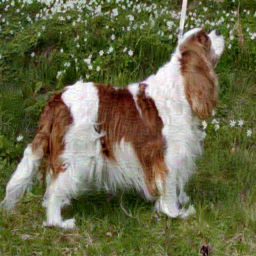}
\end{subfigure} &
% \hspace{-2mm}
% \\ \vspace{-1mm}
\begin{subfigure}[t]{0.123\textwidth}
\includegraphics[width=\textwidth]{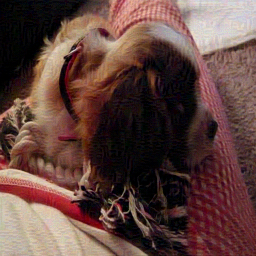}
\end{subfigure} &
% \hspace{-2mm}
% \hspace{-1.8mm}
\begin{subfigure}[t]{0.123\textwidth}
\includegraphics[width=\textwidth]{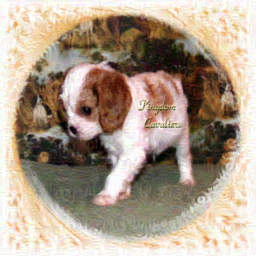}
\end{subfigure} &
% \hspace{-2mm}
\begin{subfigure}[t]{0.123\textwidth}
\includegraphics[width=\textwidth]{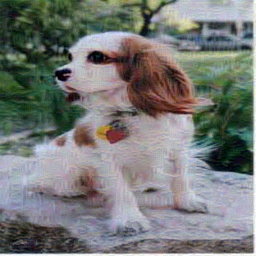}
\end{subfigure} &
\begin{subfigure}[t]{0.123\textwidth}
\includegraphics[width=\textwidth]{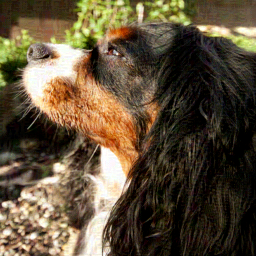}
\end{subfigure} &
% \hspace{-2mm}
% \hspace{-1.8mm}
\begin{subfigure}[t]{0.123\textwidth}
\includegraphics[width=\textwidth]{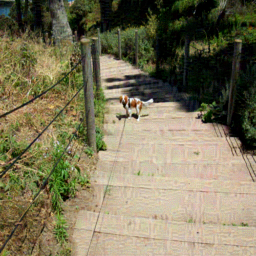}
\end{subfigure} &
\hspace{-2mm}
\begin{subfigure}[t]{0.123\textwidth}
\includegraphics[width=\textwidth]{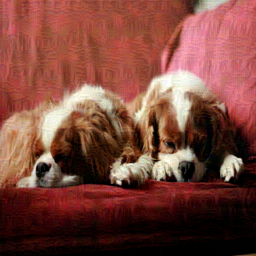}
\end{subfigure}
\\
\begin{subfigure}[t]{0.123\textwidth}
\includegraphics[width=\textwidth]{png/jsa/0_iter_100_inputs_0_jsa_diff.png}
\end{subfigure} &
% \hspace{-2mm}
\begin{subfigure}[t]{0.123\textwidth}
\includegraphics[width=\textwidth]{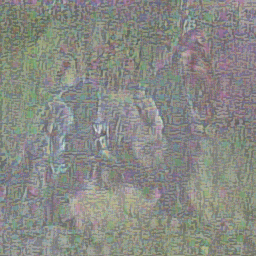}
\end{subfigure} &
% \hspace{-2mm}
% \\ \vspace{-1mm}
\begin{subfigure}[t]{0.123\textwidth}
\includegraphics[width=\textwidth]{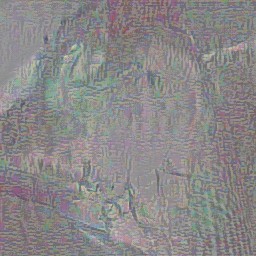}
\end{subfigure} &
% \hspace{-2mm}
% \hspace{-1.8mm}
\begin{subfigure}[t]{0.123\textwidth}
\includegraphics[width=\textwidth]{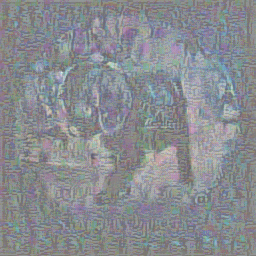}
\end{subfigure} &
% \hspace{-2mm}
\begin{subfigure}[t]{0.123\textwidth}
\includegraphics[width=\textwidth]{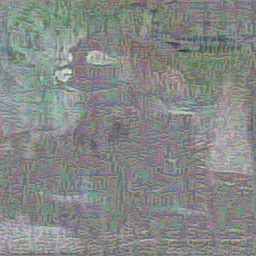}
\end{subfigure} &
\begin{subfigure}[t]{0.123\textwidth}
\includegraphics[width=\textwidth]{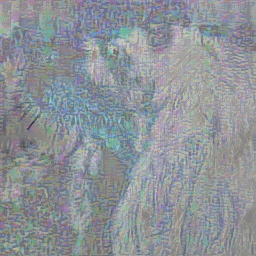}
\end{subfigure} &
% \hspace{-2mm}
% \hspace{-1.8mm}
\begin{subfigure}[t]{0.123\textwidth}
\includegraphics[width=\textwidth]{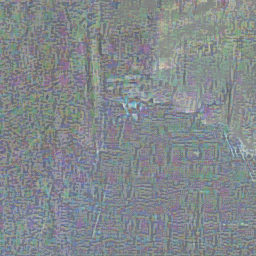}
\end{subfigure} &
% \hspace{-2mm}
\begin{subfigure}[t]{0.123\textwidth}
\includegraphics[width=\textwidth]{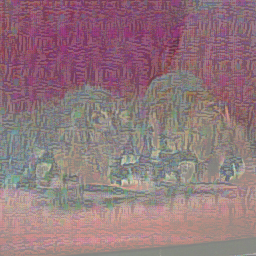}
\end{subfigure}
\vspace{2mm}
\\
\begin{subfigure}[t]{0.123\textwidth}
\includegraphics[width=\textwidth]{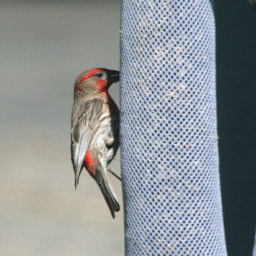}
\end{subfigure} &
% \hspace{-2mm}
\begin{subfigure}[t]{0.123\textwidth}
\includegraphics[width=\textwidth]{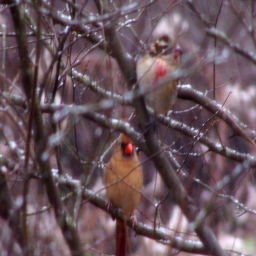}
\end{subfigure} &
% \hspace{-2mm}
% \\ \vspace{-1mm}
\begin{subfigure}[t]{0.123\textwidth}
\includegraphics[width=\textwidth]{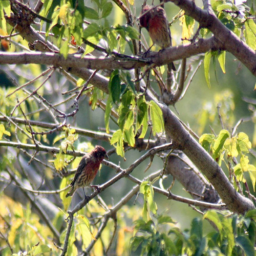}
\end{subfigure} &
% \hspace{-2mm}
% \hspace{-1.8mm}
\begin{subfigure}[t]{0.123\textwidth}
\includegraphics[width=\textwidth]{png/jsa/0_iter_200_inputs_3.png}
\end{subfigure} &
% \hspace{-2mm}
\begin{subfigure}[t]{0.123\textwidth}
\includegraphics[width=\textwidth]{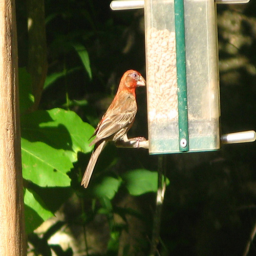}
\end{subfigure} &
\begin{subfigure}[t]{0.123\textwidth}
\includegraphics[width=\textwidth]{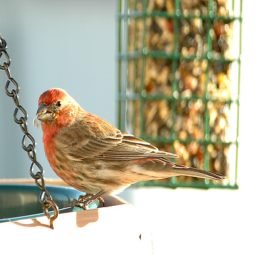}
\end{subfigure} &
% \hspace{-2mm}
% \hspace{-1.8mm}
\begin{subfigure}[t]{0.123\textwidth}
\includegraphics[width=\textwidth]{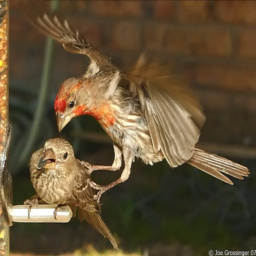}
\end{subfigure} &
\hspace{-2mm}
\begin{subfigure}[t]{0.123\textwidth}
\includegraphics[width=\textwidth]{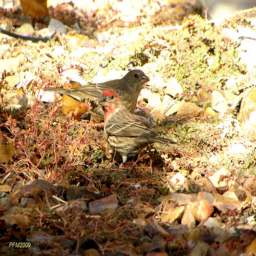}
\end{subfigure}
\\
\begin{subfigure}[t]{0.123\textwidth}
\includegraphics[width=\textwidth]{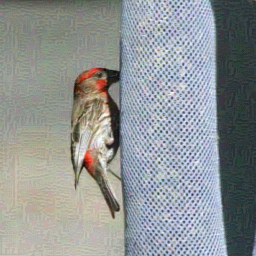}
\end{subfigure} &
% \hspace{-2mm}
\begin{subfigure}[t]{0.123\textwidth}
\includegraphics[width=\textwidth]{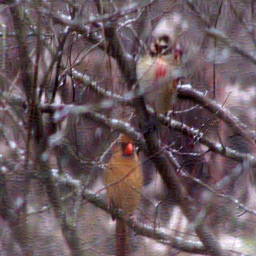}
\end{subfigure} &
% \hspace{-2mm}
% \\ \vspace{-1mm}
\begin{subfigure}[t]{0.123\textwidth}
\includegraphics[width=\textwidth]{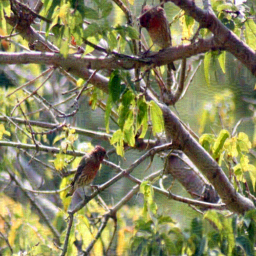}
\end{subfigure} &
% \hspace{-2mm}
% \hspace{-1.8mm}
\begin{subfigure}[t]{0.123\textwidth}
\includegraphics[width=\textwidth]{png/jsa/0_iter_200_inputs_3_jsa.png}
\end{subfigure} &
% \hspace{-2mm}
\begin{subfigure}[t]{0.123\textwidth}
\includegraphics[width=\textwidth]{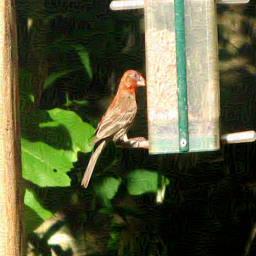}
\end{subfigure} &
\begin{subfigure}[t]{0.123\textwidth}
\includegraphics[width=\textwidth]{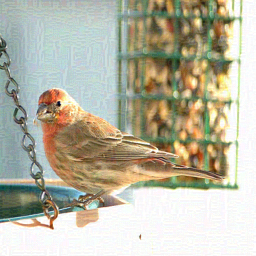}
\end{subfigure} &
% \hspace{-2mm}
% \hspace{-1.8mm}
\begin{subfigure}[t]{0.123\textwidth}
\includegraphics[width=\textwidth]{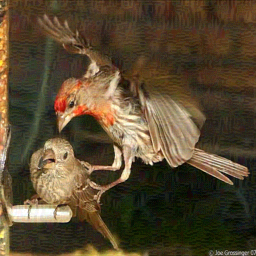}
\end{subfigure} &
\hspace{-2mm}
\begin{subfigure}[t]{0.123\textwidth}
\includegraphics[width=\textwidth]{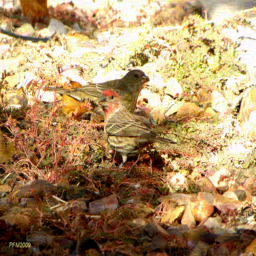}
\end{subfigure}
\\
\begin{subfigure}[t]{0.123\textwidth}
\includegraphics[width=\textwidth]{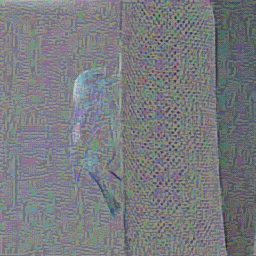}
\end{subfigure} &
% \hspace{-2mm}
\begin{subfigure}[t]{0.123\textwidth}
\includegraphics[width=\textwidth]{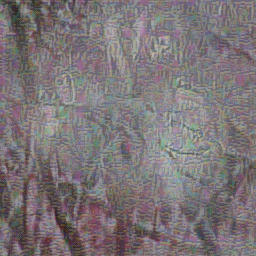}
\end{subfigure} &
% \hspace{-2mm}
% \\ \vspace{-1mm}
\begin{subfigure}[t]{0.123\textwidth}
\includegraphics[width=\textwidth]{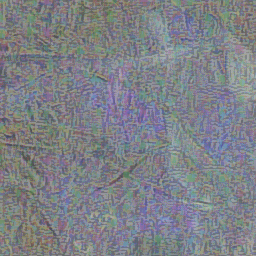}
\end{subfigure} &
% \hspace{-2mm}
% \hspace{-1.8mm}
\begin{subfigure}[t]{0.123\textwidth}
\includegraphics[width=\textwidth]{png/jsa/0_iter_200_inputs_3_jsa_diff.png}
\end{subfigure} &
% \hspace{-2mm}
\begin{subfigure}[t]{0.123\textwidth}
\includegraphics[width=\textwidth]{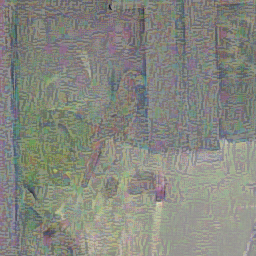}
\end{subfigure} &
\begin{subfigure}[t]{0.123\textwidth}
\includegraphics[width=\textwidth]{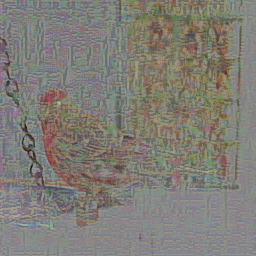}
\end{subfigure} &
% \hspace{-2mm}
% \hspace{-1.8mm}
\begin{subfigure}[t]{0.123\textwidth}
\includegraphics[width=\textwidth]{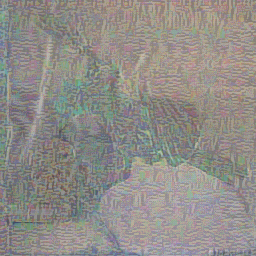}
\end{subfigure} &
\hspace{-2mm}
\begin{subfigure}[t]{0.123\textwidth}
\includegraphics[width=\textwidth]{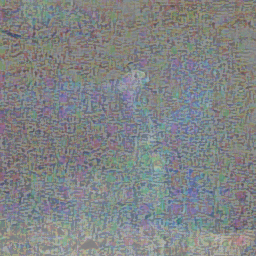}
\end{subfigure}
\end{tabular}}
% \vspace{-1mm}
% \vspace{-3mm}
\caption{Visualization of adversarial samples from JSA. Top: Original. Middle: JSA. Bottom: Magnified difference.}
\label{fig: JSA attack samples supp}
% \vspace{-2mm}
\end{figure*}

\end{document}